\documentclass{article}

\usepackage{microtype}
\usepackage{graphicx}
\usepackage{subfigure}
\usepackage{booktabs} % for professional tables

\usepackage{hyperref}

\usepackage{hyperref,graphicx,amsmath,amsfonts,amssymb,bm,url,breakurl,epsf,color,caption,multirow,comment, boldline}
\usepackage{amssymb}

\usepackage{url}            % simple URL typesetting
\usepackage{booktabs}       % professional-quality tables
\usepackage{amsfonts}       % blackboard math symbols

\usepackage{wrapfig}

\definecolor{darkred}{RGB}{150,0,0}
\definecolor{darkgreen}{RGB}{0,150,0}
\definecolor{darkblue}{RGB}{0,0,200}
\hypersetup{colorlinks=true, linkcolor=darkred, citecolor=darkgreen, urlcolor=darkblue}

\newcommand{\beq}{\begin{equation}}

\newcommand{\eeq}{\end{equation}}

\newcommand{\x}{\vct{x}}

\definecolor{emmanuel}{RGB}{255,127,0}

\newcommand{\vct}[1]{\bm{#1}}
\newcommand{\mtx}[1]{\bm{#1}}

\newcommand{\X}{{\mtx{X}}}

\numberwithin{equation}{section} 

\usepackage[accepted]{icml2023}
\usepackage{amsmath}
\usepackage{amssymb}
\usepackage{mathtools}
\usepackage{amsthm}

\newtheorem{remark}{Remark}
\newtheorem{theorem}{Theorem}
\newtheorem{lemma}{Lemma}

\theoremstyle{plain}

\icmltitlerunning{Adversarial Training with Generated Data in High-Dimensional Regression: An Asymptotic Study}

\begin{document}

\twocolumn[
\icmltitle{Adversarial Training with Generated Data in High-Dimensional Regression: An Asymptotic Study}



\icmlsetsymbol{equal}{*}

\begin{icmlauthorlist}
\icmlauthor{Yue Xing}{yyy}
\end{icmlauthorlist}

\icmlaffiliation{yyy}{Department of Statistics and Probability, Michigan State University, United States}

\icmlcorrespondingauthor{Yue Xing}{xingyue1@msu.edu}

\icmlkeywords{Machine Learning, ICML}

\vskip 0.3in
]



\printAffiliationsAndNotice{}  % leave blank if no need to mention equal contribution

\begin{abstract}
In recent years, studies such as \cite{carmon2019unlabeled,gowal2021improving,xing2022artificial} have demonstrated that incorporating additional real or generated data with pseudo-labels can enhance adversarial training through a two-stage training approach. In this paper, we perform a theoretical analysis of the asymptotic behavior of this method in high-dimensional linear regression. While a double-descent phenomenon can be observed in ridgeless training, with an appropriate $\mathcal{L}_2$ regularization, the two-stage adversarial training achieves a better performance. Finally, we derive a shortcut cross-validation formula specifically tailored for the two-stage training method.
\end{abstract}

\section{Introduction}\label{sec:intro}

The development of machine learning and deep learning methods has led to breakthrough performance in various applications. However, recent studies, e.g., \cite{goodfellow2014explaining}, observe that these models are vulnerable when the data are perturbed by adversaries. Attacked inputs can be imperceptibly different from clean inputs to humans but can cause the model to make incorrect predictions. 

To defend against adversarial attacks, adversarial training is a popular and promising way to improve the adversarial robustness of modern machine learning models. Adversarial training first generates attacked samples, then calculates the gradient of the model based on these augmented data. Such a procedure can make the model less susceptible to adversarial attacks in real-world situations.

There are fruitful results in the theoretical justification and methodology development in adversarial training. Among various research directions, one interesting aspect is to improve adversarial training with extra unlabeled data. Recent works successfully demonstrate great improvements in the adversarial robustness with additional unlabeled data. For example, \cite{xing2021adversarially}, show that additional external real data help improve adversarial robustness; \cite{gowal2021improving,wang2023better} use synthetic data to improve the adversarial robustness and achieve the highest 65\% to 70\% adversarial testing accuracy for CIFAR-10 dataset under AutoAttack (AA) in \cite{croce2020robustbench}\footnote{ \url{https://robustbench.github.io}  }.

A recent study \cite{xing2022artificial} reveals that adversarial training gains greater benefits from unlabeled data than clean (natural) training. The key observation is that adversarially robust models rely on the conditional distribution of the response given the features ($Y|X$) and the marginal distribution of the features ($X$). In contrast, clean training only depends on $Y|X$ in their study. As a result, adversarial training can benefit more than clean training from unlabeled data.

Besides adversarial training, high dimensional statistics is another important field of traditional machine learning to solve real-world problems from genomics, neuroscience to image processing.
While many studies focus on obtaining a better performance via regularization, one surprising phenomenon in this field is the double descent phenomenon \citep{belkin2019two,hastie2019surprises}, which refers to a U-shaped curve in the test error as a function of the model complexity, together with a second descent phase occurring in the over-parameterized regime. This phenomenon challenges the conventional wisdom that increasing model complexity always leads to over-fitting. It provides significant implications for designing and analyzing machine learning algorithms in high-dimensional settings. 

Given the substantial achievements in high-dimensional statistics, this paper aims to extend the analysis of \cite{xing2022artificial} to a high-dimensional regression setup, in which both the data dimension $d$ and the sample size of the labeled data $n_1$ increase and $d/n_1\rightarrow\gamma$ asymptotically. Although \cite{xing2022artificial} provides a theoretical explanation for the benefits of unlabeled data in the large sample regime ($n_1\gg d$), the asymptotic behavior of the two-stage method in other scenarios remains unclear.

Our contributions are summarized as follows:
\begin{itemize}
	\item We derived the asymptotic convergence of the two-stage adversarial training when $d/n_1\rightarrow\gamma$ for some constant $\gamma>0$. (Section \ref{sec:analysis}).
	\item It is observed that a proper ridge penalty in the clean training stage benefits the two-stage method. However, the optimal ridge penalty for the clean estimate in the first stage of \cite{xing2022artificial} differs from the one yielding the best clean performance. We conjecture that this discrepancy arises from the change in the error decomposition from clean training to two-stage adversarial training. To facilitate more efficient hyper-parameter tuning, we propose adaptations to existing cross validation (CV) methods, improving the time-consuming vanilla CV approach (Sections \ref{sec:weight_decay} and \ref{sec:cross_validation}).
\end{itemize}

\subsection{Related Works}

Below is a summary of related works in adversarial training, high-dimensional statistics, and cross validation.

\paragraph{Adversarial Training.}

% A model that is robust against adversarial attacks should be able to make accurate predictions even when the distribution of test data has been adversarially altered. Despite numerous attempts to enhance adversarial robustness, existing algorithms still fall short of meeting the desired theoretical performance. For instance, on the well-established CIFAR-10 dataset, which comprises 60,000 samples across ten different classes, the best state-of-the-art algorithms using large neural networks can only achieve a maximum of 67\% adversarial testing accuracy against an 8/255 $\mathcal{L}_{\infty}$ attack, according to the standardized benchmark RobustBench \citep[][available on the GitHub platform]{croce2020robustbench}.

% The research in the area of adversarial learning has primarily focused on providing algorithmic advancements through heuristics drawn from domain knowledge or advanced engineering. Recently, some studies examine the statistical properties of adversarial training and robust models \citep[e.g.,][]{javanmard2020precise,javanmard2020precise1,xing2021adversarially,dan2020sharp,sinha2018certifying,yin2018rademacher,allen2020feature,xiao2022adversarial}. Despite these advancements, a considerable gap persists between theoretical predictions and empirical observations.

There are many studies in the area of adversarial training. Some studies, e.g., \cite{goodfellow2014explaining,zhang2019theoretically,wang2019improving,cai2018curriculum,zhang2020attacks,carmon2019unlabeled,gowal2021improving}, work in methodology.
Theoretical investigations have also been conducted from different perspectives. For instance, \citet{chen2020more, javanmard2020precise,taheri2021statistical,yin2018rademacher,raghunathan2019adversarial,najafi2019robustness,min2020curious,hendrycks2019using,dan2020sharp,wu2020revisiting,deng2021improving} study the statistical properties of adversarial training; \citet{sinha2018certifying,wang2019convergence,xiao2022stability} study the optimization perspective; \citet{gao2019convergence,zhang2020over,zhang2023understanding,mianjy2022robustness,lvimplicit,xiao2021adversarial} work on  deep learning.

\paragraph{Double Descent and High-Dimensional Statistics.}

Double descent phenomenon is an observation in the learning curves of machine learning models. It describes the behavior of the generalization gap, i.e., the difference between the model performance on the training data and testing data. In a typical learning curve, the generalization error decreases and then increases with larger model complexity. However, in the double descent phenomenon, after the first decrease-increase pattern, the error decreases again when further enlarging the model complexity in the over-fitting regime. This non-monotonic behavior of the learning curve has been observed in various machine learning settings. Comprehensive investigations into the double descent phenomenon can be found in \cite{belkin2019two,hastie2019surprises,ba2020generalization,d2020double,adlam2020understanding,liu2021kernel,rocks2022memorizing}.

\paragraph{Cross Validation.}

Cross validation (CV) is a resampling procedure used to evaluate the performance of machine learning models. This paper mainly considers leave-one-out CV. For leave-one-out CV, it trains the model using all-but-one samples and repeats this process so that every sample is left in the estimation once. The final model performance is then averaged across all the models. The model can generalize better to new data by optimizing the hyperparameters in the model, e.g., regularization, through CV.

However, although a leave-one-out CV is an effective method for selecting hyperparameters, it is time-consuming by its design. Consequently, some studies propose shortcut formulas for the leave-one-out CV to reuse some terms when estimating the model using different data. Studies related to CV can be found in
\cite{stone1978cross,picard1984cross,shao1993linear,browne2000cross,berrar2019cross}.

\section{Model Setup}\label{sec:assumptions}

In this section, we present the data generation model and the two-stage adversarial training framework.

 \paragraph{Data generation model.}
We assume that the attributes $X\sim N(\textbf{0},\vct{\vct\Sigma})$ with covariance matrix $\vct\Sigma=\vct I_d$, and the response $Y$ satisfies $Y=X^{\top}{\vct\theta}_0+\varepsilon$ for $\|{\vct\theta}_0\|=r=O(1)$ and a Gaussian noise $\varepsilon$ with $Var(\varepsilon)=\sigma^2$. 

% \textbf{Adversarial training} Adversarial training is a generic training algorithm, which can be utilized either in linear regression solely or in the two-stage training algorithm. Adversarial training aims to minimize the empirical adversarial loss below:
% \begin{eqnarray}
%     \frac{1}{n_1}\sum_{i=1}^{n_1} \sup_{z\in\mathcal{B}_2(x_i,\epsilon)} (z^{\top}{\vct\theta}-y_i)^2.\label{eqn:vanilla_adv}
% \end{eqnarray}
% and denote $\widehat{{\vct\theta}}_\epsilon(\lambda)$ as the corresponding solution. 
% Both the vanilla adversarial training and the two-stage method below aim to estimate the population adversarial risk minimizer
% \begin{eqnarray*}
%     {\vct\theta}_\epsilon=\arg\min \mathbb{E} \sup_{z\in\mathcal{B}_2(X,\epsilon)} (z^{\top}{\vct\theta}-Y)^2:=\arg\min_{\vct\theta} R_\epsilon({\vct\theta},\vct\theta_0).
% \end{eqnarray*}
\paragraph{Two-stage adversarial training.} There are two stages in this training framework. In the first stage, we utilize $n_1$ i.i.d. labeled samples, i.e., $(\x_i,y_i)$ for $i=1,\ldots, n_1$. We consider the scenario where $d\asymp n_1$. The first stage solves the following clean training problem
\begin{eqnarray}
    \frac{1}{n_1}\sum_{i=1}^{n_1}  (\vct x_i^{\top }{\vct\theta}-y_i)^2+\lambda\|{\vct\theta}\|^2\label{eqn:clean}
\end{eqnarray}
and obtain the clean estimate $\widehat{{\vct\theta}}_0(\lambda)$. 

In the second stage, we use the trained model $\widehat{{\vct\theta}}_0(\lambda)$ to generate a pseudo response for a set of unlabeled data, i.e., 
\begin{eqnarray*}
    \widehat{y}_i=\vct x_i^{\top }\widehat{{\vct\theta}}_0(\lambda)+\varepsilon_i
\end{eqnarray*}
for $i=n_1+1,\ldots,n_1+n_2$. In this paper, we consider the scenario where $n_2=\infty$. We also assume $\sigma^2$ is known and $\varepsilon_i$ are generated from $N(0,\sigma^2)$. Finally we use the extra data with pseudo response to do adversarial training and minimize the following loss w.r.t $\vct\theta$:
\begin{eqnarray}
    \frac{1}{n_2}\sum_{i=n_1+1}^{n_1+n_2}  \sup_{\vct z\in\mathcal{B}_2(\x_i,\epsilon)} (\vct z^{\top}{\vct\theta}-
\widehat{y}_i)^2.\label{eqn:adv}
\end{eqnarray}
Denote the final solution as  $\vct{\widetilde\theta}_\epsilon(\lambda)$.

\begin{remark}
    The two-stage method in this paper is slightly different from the original one in \cite{gowal2021improving,xing2022artificial}. We only utilize the generated data in the second stage. This simplifies the theoretical analysis. In addition, when $d/n_1=\gamma$ is a large constant, we empirically observe that the two-stage method is better than an adversarial training with only labeled data, i.e., the right of Figure \ref{fig:compare}.
\end{remark}

\begin{remark}
Our initial trial indicates that adding additional regularization in equation (\ref{eqn:adv}) does not help much. Thus, we only inject a penalty in the clean training stage.
\end{remark}

% The penalty term $\lambda$ in (\ref{eqn:vanilla_adv}) and (\ref{eqn:clean}) and the $\eta$ in (\ref{eqn:adv}) can be different.

\textbf{Expected Adversarial Risk} Under the model assumption of $(X,Y)$, the population adversarial risk for any given estimate $\vct\theta$ becomes
\begin{eqnarray*}
    R_\epsilon(\vct\theta,\vct\theta_0)
    &=&\|\vct\theta-\vct\theta_0\|_{\vct\Sigma}^2\\
    &&+2c_0\epsilon\|\vct\theta\|\sqrt{\|\vct\theta-\vct\theta_0\|_{\vct\Sigma}^2+\sigma^2}+\epsilon^2\|\vct\theta\|^2,
\end{eqnarray*}
where $\|\cdot\|$ is the $\mathcal{L}_2$ norm, and $c_0=\sqrt{2/\pi}$ is derived from the exact distribution of $(X,Y)$. We rewrite $R_\epsilon(\vct\theta,\vct\theta_0)$ as $R_\epsilon(\vct\theta)$ for simplicity when no confusion arises.

\begin{remark}
    One can denote $\vct\theta_\epsilon=\arg\min_{\vct\theta} R_\epsilon(\vct\theta,\vct\theta_0)$
as the best robust model. However, from $R_\epsilon(\vct\theta,\vct\theta_0)$, we are interested in $\|\vct\theta-\vct\theta_0\|_{\vct\Sigma}$ and $\|\vct\theta\|$ rather than $\|\vct\theta-\vct\theta_\epsilon\|$. 

    Based on \cite{xing2021adversarially}, when an estimate $\vct\theta\rightarrow\vct\theta_\epsilon$, the excess adversarial risk $R_\epsilon(\vct\theta,\vct\theta_0)-R_\epsilon(\vct\theta_\epsilon,\vct\theta_0)$ can be approximated by a function of $\vct\theta-\vct\theta_\epsilon$. However, when $\vct\theta-\vct\theta_\epsilon$ diverges in the high-dimensional setup, such an approximation leads to a large error.
\end{remark}

\section{Analyzing the Two-Stage Adversarial Training Framework}

This section presents the main theoretical results and simulation studies. We first demonstrate the main theory of the convergence of the two-stage method in Section \ref{sec:analysis}, take different $\lambda$ under different attack strength $\epsilon$ in Section \ref{sec:weight_decay}, and finally introduce a CV method in Section \ref{sec:cross_validation}.

\subsection{Convergence Result}\label{sec:analysis}

For the two-stage adversarial framework, to study $\vct{\widetilde\theta}_{\epsilon}(\lambda)$, we denote the following function
\begin{eqnarray*}
    m_{\gamma}(-\lambda) = \frac{-(1-\gamma+\lambda)+\sqrt{ (1-\gamma+\lambda)^2+4\lambda\gamma}}{2\gamma\lambda},
\end{eqnarray*}
which is used to describe the asymptotic behavior of $tr\left( (\sum_{i=1}^{n_1}\x_i\x_i^{\top}+\lambda \vct{I}_d)^{-1} \right)$ as in \cite{hastie2019surprises}.

After defining $m_{\gamma}$, one can obtain the convergence of $\widehat{\vct\theta}_0(\lambda)$, and further figure out the asymptotic behavior of $\widehat{\vct\theta}_\epsilon(\lambda)$. The convergence of the two-stage adversarial training framework is as follows:
\begin{theorem}[Convergence of Two-Stage Adversarial Training]\label{thm:adv_convergence}
    With probability tending to 1, $\widehat{\vct\theta}_0(\lambda)$ satisfies
% \begin{eqnarray*}
%    \mathbb{E}\|\widehat{{\vct\theta}}_0(\lambda)-{\vct\theta}_0\|^2 = \frac{\lambda^2}{n^2}m_{\gamma}'(-\lambda) + \sigma^2\gamma\left( -\lambda)-\frac{\lambda}{n_1}m_{\gamma}'(-\lambda)\right),
% \end{eqnarray*}
	\begin{eqnarray*}
	\|\vct{\widehat{\theta}}_0(\lambda)-\vct\theta_0\|^2&\rightarrow&\lambda^2 r^2 m_{\gamma}'(-\lambda)\\&&+\sigma^2\gamma\left(m_{\gamma}(-\lambda)-\lambda m_{\gamma}'(-\lambda) \right),\\
	\|\vct{\widehat{\theta}}_0(\lambda)\|^2
 &\rightarrow&r^2[ 1-2\lambda m_{\gamma}(-\lambda)+\lambda^2m_{\gamma}'(-\lambda)  ]\\&&+\sigma^2\gamma[m_{\gamma}(-\lambda)-\lambda m_{\gamma}'(-\lambda)].
	\end{eqnarray*}
% where $\psi_{1,d,n}(\lambda):=\mathbb{E}(\vct{X}_n^{\top}\vct{X}_n+n\lambda \vct{I}_d)^{-1}$, $\psi_{2,d,n}(\lambda):=\mathbb{E}(\vct{X}_n^{\top}\vct{X}_n+n\lambda \vct{I}_d)^{-2}$, and
% it satisfies that
% \begin{eqnarray*}
%   tr(\psi_{1,d,n}(\lambda)) = \gamma -\lambda), tr(\psi_{2,d,n}(\lambda)) =\frac{\gamma}{n_1} m'_{\gamma}(-\lambda).
% \end{eqnarray*}
% In addition, denote $(\alpha,\vct\theta_{adv}, \vct\theta_{base})\in\{(\alpha_\epsilon(\lambda,0),\vct{\widetilde\theta}_\epsilon(\lambda,0),\vct{\widehat\theta}_0(\lambda)),(\alpha_\epsilon^*,\vct{\theta}_\epsilon,\vct\theta_0)\}$ as
% \begin{eqnarray*}
%    \vct\theta_{adv} = (\vct{\vct\Sigma}+\alpha \vct{I}_d)^{-1}\vct\theta_{base}:=\arg\min_{\vct\theta}R_\epsilon(\vct\theta,\vct\theta_{base}).
% \end{eqnarray*}
For the two-stage adversarial estimate $\vct{\widetilde\theta}_\epsilon(\lambda)$, assuming $n_2=\infty$, $\widetilde{\vct\theta}_\epsilon(\lambda)$ satisfies
\begin{eqnarray*}
    \|\widetilde{\vct\theta}_\epsilon(\lambda)-{\vct\theta}_0\|^2
    &\rightarrow&\frac{1}{(1+\alpha_\epsilon(\lambda))^2}\|\widehat{\vct\theta}_0(\lambda)\|^2\\
    &&+r^2-\frac{2}{(1+\alpha_\epsilon(\lambda))}\widehat{\vct\theta}_0(\lambda)^{\top}{\vct\theta}_0,\\
    \|{\vct{\widetilde\theta}}_\epsilon(\lambda)\|^2&\rightarrow& \frac{1}{(1+\alpha_\epsilon(\lambda))^2}\|\widehat{\vct\theta}_0(\lambda)\|^2,
\end{eqnarray*}
where $2\widehat{\vct\theta}_0(\lambda)^{\top}{\vct\theta}_0$ can be calculated via $$2\widehat{\vct\theta}_0(\lambda)^{\top}{\vct\theta}_0=\|\vct\theta_0\|^2+\|\widehat{\vct\theta}_0(\lambda)\|^2-\|\widehat{\vct\theta}_0(\lambda)-\vct\theta_0\|^2,$$ and $\alpha_\epsilon(\lambda)$ is the solution of $\alpha$ in 
\begin{eqnarray*}
    &&\alpha+\epsilon c_0 \frac{\alpha\|\widehat{\vct\theta}_0(\lambda)\|}{\sqrt{ \|\widehat{\vct\theta}_0(\lambda)\|^2\alpha^2+\sigma^2(1+\alpha)^2 }}\\
    &=&\epsilon c_0 \frac{\sqrt{ \|\widehat{\vct\theta}_0(\lambda)\|^2\alpha^2+\sigma^2(1+\alpha)^2 }}{\|\widehat{\vct\theta}_0(\lambda)\|}+\epsilon^2.
\end{eqnarray*}
% \begin{eqnarray*}
%     &&\|\widetilde{\vct\theta}_\epsilon(\lambda,0)-{\vct\theta}_\epsilon\|^2\\
%     &\rightarrow&\frac{1}{(1+\alpha_\epsilon(\lambda,0))^2}\|\widehat{\vct\theta}_0(\lambda)\|^2+\frac{1}{(1+\alpha^*)^2}\|{\vct\theta}_0\|^2-2\frac{1}{(1+\alpha_\epsilon(\lambda,0))(1+\alpha_\epsilon^*)}\widehat{\vct\theta}_0(\lambda)^{\top}{\vct\theta}_0,
% \end{eqnarray*}
% where 
% When ${\vct\theta}_0\sim N(\textbf{0},\vct{I}_d/d)$,
% \begin{eqnarray*}
%     \widehat{\vct\theta}_0(\lambda)^{\top}{\vct\theta}_0
%     &=&{\vct\theta}_0^{\top}\vct{X}_n^{\top}\vct{X}_n(\vct{X}_n^{\top}\vct{X}_n+n\lambda \vct{I}_d)^{-1}{\vct\theta}_0\\
%     &&+\vct{\epsilon}^\top\vct{X}_n(\vct{X}_n^{\top}\vct{X}_n+n\lambda \vct{I}_d)^{-1}{\vct\theta}_0\\
%     &\rightarrow& 1-\lambda \gamma m_{\gamma}(-\lambda).
% \end{eqnarray*}
\end{theorem}

% Different from Theorem \ref{thm:vanilla_ridge}, for the two-stage adversarial training, we need the convergence of $\widehat{\vct\theta}_0(\lambda)$ and utilize it in the second stage.
The proof of Theorem \ref{thm:adv_convergence} is in the appendix. We first study the convergence of $\vct{\widehat\theta}_0(\lambda)$, and then evaluate $\vct{\widetilde\theta}_\epsilon(\lambda)$.

From Theorem \ref{thm:adv_convergence}, similar to $\vct{\widehat\theta}_0$, one can see that $\|\widetilde{\vct\theta}_\epsilon(\lambda)-{\vct\theta}_0\|^2$ and $\|\widetilde{\vct\theta}_\epsilon(\lambda)\|^2$ converges to some value as a function of $(\gamma,\lambda,\epsilon,\sigma^2)$ asymptotically. 

% We only consider $\eta=0$ in Theorem \ref{thm:adv_convergence}. For other $\eta>0$, since $\vct\Sigma=\vct I_d$, it is equivalent to taking a different $\epsilon$ during the training stage. 
% Discussions can be found in Section \ref{sec:attack}, and changing $\eta$ or $\epsilon$ not have great impact in the final adversarial testing performance.

We conduct a simulation to verify Theorem \ref{thm:adv_convergence} and study the risk of the two-stage adversarial training. In the experiment, we take $n_1=100$ and $n_2=\infty$, i.e., we directly use the population adversarial risk in the second stage. We change the data dimension $d$ to obtain different $\gamma=d/n_1$. The data follows $X\sim N(\textbf{0},\vct I_d)$, $Y=X^{\top}{\vct\theta}_0+\varepsilon$ with ${\vct\theta}_0\sim N(0,\vct I_d/d)$ and $\varepsilon\sim N(0,1)$. The adversarial attack is taken as $\epsilon=0.3$. We repeat the experiment 100 times to obtain the average performance. We use the excess adversarial risk, i.e., $R_\epsilon({\vct\theta})-R_\epsilon({\vct\theta}_\epsilon)$ for ${\vct\theta}\in\{\widehat{\vct\theta}_0(\lambda),\widehat{\vct\theta}_\epsilon(\lambda),\widetilde{\vct\theta}_\epsilon(\lambda)\}$, to evaluate the performance of the three methods. The model $\widehat{\vct\theta}_\epsilon(\lambda)$ refers to the vanilla adversarial training as an additional benchmark, i.e., we conduct adversarial training using the $n_1$ labeled samples. The simulation results are summarized in Figure \ref{fig:compare}, \ref{fig:compare_theory}, \ref{fig:ridge_vs_ridgeless}, \ref{fig:ridge_vs_ridgeless_theory}.

\begin{figure}[!ht]
    \centering
    \includegraphics[scale=0.45]{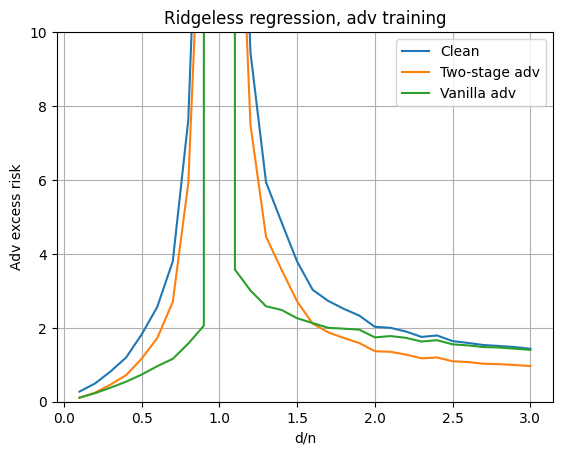}
    \caption{Simulation: Excess adversarial risk of clean training, vanilla adversarial training, and the two-stage adversarial training, without ridge penalty.}
    \label{fig:compare}
\end{figure}
\begin{figure}[!ht]
    \centering
    \includegraphics[scale=0.45]{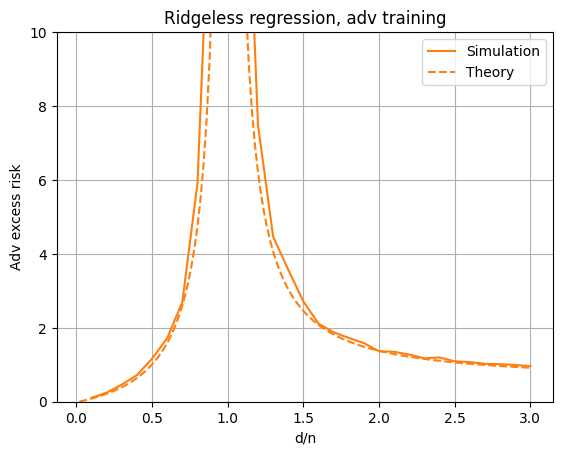}
    \caption{Theoretical value corresponding to Figure \ref{fig:compare}.}
    \label{fig:compare_theory}
\end{figure}

In Figure \ref{fig:compare}, we take $\lambda\rightarrow0$ to align with the experiments in the double descent literature.
There are several observations from Figure \ref{fig:compare}. First, if we compare the performance of the two-stage adversarial training and the clean training, the two-stage adversarial training is better than clean training. Second, when $d/n_1$ gets larger, the performance of the two-stage adversarial training is better than the vanilla adversarial training, indicating that the information of the additional extra data matters. Finally, for all the three training methods, they all observe a double-descent phenomenon.

In addition, we plot the theoretical curves for the excess adversarial risk associating with the two-stage adversarial training. From Figure \ref{fig:compare_theory}, the theoretical curve and the simulation result match with each other.

\begin{figure}[!ht]
    \centering
    \includegraphics[scale=0.45]{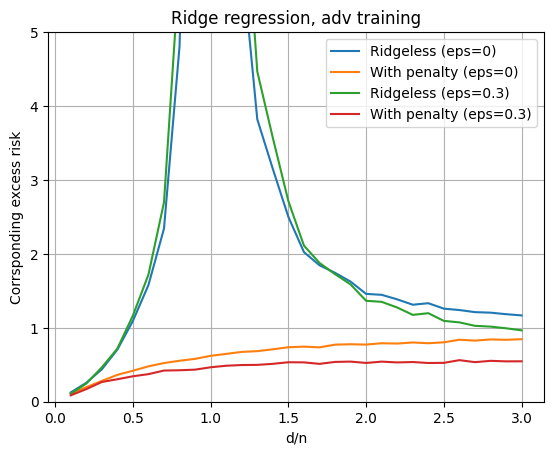}
    \caption{Simulation: Ridgeless regression and ridge regression with the best penalty in clean training and the two-stage adversarial training respectively. Adversarial training benefits more from a proper penalty.}
    \label{fig:ridge_vs_ridgeless}
\end{figure}

\begin{figure}[!ht]
    \centering
    \includegraphics[scale=0.45]{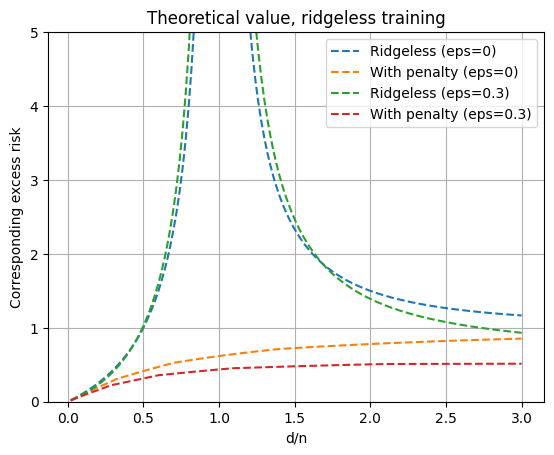}
    \caption{Theoretical value corresponding to Figure \ref{fig:ridge_vs_ridgeless}}
    \label{fig:ridge_vs_ridgeless_theory}
\end{figure}

Finally, we examine how the ridge penalty affects the performance. In the simulation in Figure \ref{fig:ridge_vs_ridgeless}, we take $\epsilon=0,0.3$ and compare the performance when $\lambda=0$ and $\lambda$ is taken to minimize the risk. In Figure \ref{fig:ridge_vs_ridgeless}, the y-axis is the corresponding excess adversarial risk, i.e., $\epsilon=0,0.3$ for the corresponding groups respectively. The corresponding theoretical curves can be found in Figure \ref{fig:ridge_vs_ridgeless_theory}.

From Figure \ref{fig:ridge_vs_ridgeless}, one can see that the excess risk for the ridgeless regression is similar, while the two-stage adversarial training ($\epsilon=0.3$) benefits more than clean training ($\epsilon=0$) when taking a proper ridge penalty, which motivates us to further investigate in the penlaty term in the following sections. In addition, the theoretical curves in Figure \ref{fig:ridge_vs_ridgeless_theory} align with the simulation results in \ref{fig:ridge_vs_ridgeless} as well.
% In addition, Figure \ref{fig:clean} shows the generalization of clean training with/without penalty. One can see that clean training benefits from a proper penalty. However, in Figure \ref{fig:compare}, one can see that the improvement of the excess adversarial risk in adversarial training gets much more reduced compared to the one in Figure \ref{fig:clean}. This indicates that adversarial training benefits more from the ridge penalty.

\subsection{A Better Clean Estimate May Not Be Preferred}\label{sec:weight_decay}

Different from ridgeless regression in the large-sample regime, with high-dimensional data, it is essential to utilize ridge penalty or other regularization to improve the testing performance. While one can adjust the penalty to control the performance of the clean estimate, we would like to ask:

\begin{center}
    \textit{Is a better clean estimate (measured by clean testing performance) always preferred in the two-stage method?}
\end{center}

To answer the above question, it is essential to investigate the role of the clean estimate in the two-stage method. Recall that the population adversarial risk is written as
\begin{eqnarray*}
    R_\epsilon(\vct\theta,\vct\theta_0)
    &=&\|\vct\theta-\vct\theta_0\|_{\vct\Sigma}^2\\
    &&+2c_0\epsilon\|\vct\theta\|\sqrt{\|\vct\theta-\vct\theta_0\|_{\vct\Sigma}^2+\sigma^2}+\epsilon^2\|\vct\theta\|^2,
\end{eqnarray*}
where taking expectation on training data we have
\begin{eqnarray*}
    \mathbb{E}\|\vct{\widetilde\theta}_\epsilon(\lambda)-\vct\theta_0\|_{\vct\Sigma}^2=\|\mathbb{E}\vct{\widetilde\theta}_\epsilon(\lambda)-\vct\theta_0\|_{\vct\Sigma}^2 + tr(Var(\vct{\widetilde\theta}_\epsilon(\lambda))),
\end{eqnarray*}
and
\begin{eqnarray*}
    \mathbb{E}\|\vct{\widetilde\theta}_\epsilon(\lambda)\|^2=\|\mathbb{E}\vct{\widetilde\theta}_\epsilon(\lambda)\|^2+tr(Var(\vct{\widetilde\theta}_\epsilon(\lambda))).
\end{eqnarray*}
The above decompositions imply that the 
% Denote $l_\epsilon(\x,\y,\vct\theta)$ as adversarial loss for each sample $(\x,\y)$. Based on \cite{xing2022artificial}, in the large sample regime, $\widehat{{\vct\theta}}(0)$ is consistent, and ${\vct\theta}_\epsilon-\widetilde{{\vct\theta}}(\epsilon)$ is dominated by
% \begin{eqnarray*}
% &&\underbrace{\frac{n_2\Sigma_{\epsilon}^{-1}}{n_1+n_2} \mathbb{E}\left(\frac{\partial }{\partial {\vct\theta}_{\epsilon}} l_{\epsilon}(X,\widehat Y,{\vct\theta}_\epsilon)  -\frac{\partial }{\partial {\vct\theta}_{\epsilon}}  l_{\epsilon}(X,Y,{\vct\theta}_\epsilon)  \right)}_{\approx\frac{n_2}{n_1+n_2}\vct\Sigma_{\epsilon}^{-1}\widetilde{\vct\Sigma}_{\epsilon}(\widehat{{\vct\theta}}(0)-{\vct\theta}_0):=E_1} \\
% &&+\underbrace{\frac{\vct\Sigma_{\epsilon}^{-1}}{n_1+n_2}\left( \sum_{S_1,S_2} \frac{\partial }{\partial {\vct\theta}_{\epsilon}} l_{\epsilon}(x,y,{\vct\theta}_{\epsilon}) \right)}_{:=E_2}.
% \end{eqnarray*}
% When $n_2\rightarrow\infty$, $E_2\rightarrow 0$. In terms of $E_1$, its formula 
% where $E_1$ is from the estimation bias of $\widehat{\vct\theta}_0(\lambda)$ and $E_2$ is from the variance.\footnote{The method from \cite{xing2022artificial} is slightly different from our scenario. They utilize both the labeled and unlabeled data in the second stage. However, when $n_2\rightarrow\infty$, the result still applies to our scenario.}
while ridge regression balances bias and variance of $\widehat{\vct\theta}_0(\lambda)$, the importance of bias and variance are changed in $\widetilde{\vct\theta}_\epsilon(\lambda)$. As a result, the optimal $\lambda$ for the clean estimate may not be the best when applied in the two-stage adversarial training.

% From the above intuition, we study Theorem \ref{thm:adv_convergence} and find that the best $\lambda$ for the two-stage adversarial training is different from the best one used in the clean ridge regression problem:
% \begin{proposition}
%     \begin{eqnarray*}
%         \arg\min_{\lambda} \mathbb{E}\|\widetilde{\vct\theta}_\epsilon(\lambda,0)-{\vct\theta}_\epsilon\|^2_2\geq \arg\min_{\lambda} \mathbb{E}\|\widehat{\vct\theta}_0(\lambda,0)-{\vct\theta}_0\|^2_2.
%     \end{eqnarray*}
% \end{proposition}
To investigate how the optimal $\lambda$ changes in the two-stage method, a simulation study is conducted in Figure \ref{fig:ridge_regression}. We take $n_1=50$. The data $X\sim N(\textbf{0},\vct I_d)$ and $d=200$. The response $Y={\vct\theta}_0^{\top}X+\varepsilon$ with ${\vct\theta}_0=\textbf{1}/\sqrt{d}$ and $\varepsilon\sim N(0,0.1^2)$. Besides the $n_1$ labeled data, we take  $n_2=\infty$. We repeat 30 times to get the average result and check the best $\lambda$ under different attack strength $\epsilon$. 

\begin{figure}
    \centering
    \includegraphics[scale=0.45]{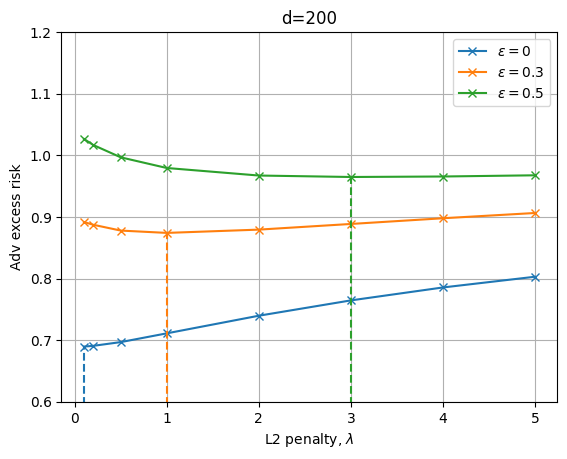}
    \caption{Simulation: How the tuning parameter $\lambda$ in clean ridge regression affects the final adversarial robustness when using extra unlabeled data in training. While a small $\lambda$ minimizes the population clean risk, this choice of $\lambda$ is sub-optimal when using $\widehat{{\vct\theta}}_0(\lambda)$ to create pseudo response. Besides the cases of $\epsilon\in\{0,0.3,0.5\}$, when $\epsilon=0.7$, the best penalty $\lambda$ is extremely large and is not included in the figure. }
    \label{fig:ridge_regression}
\end{figure}

From Figure \ref{fig:ridge_regression}, one can see that the optimal $\lambda$ gets larger when the attack strength gets larger. When $\epsilon=0$, the optimal $\lambda$ is closed to zero. When $\epsilon=0.3$, the best $\lambda$ is around 1, and 3 when $\epsilon=0.5$, both of which are much larger than the case for $\epsilon=0$. 

\subsection{Cross Validation}\label{sec:cross_validation} 
Observing that the optimal $\lambda$ for clean training is not the best for the two-stage adversarial training, we next investigate how to better select a proper $\lambda$.

While one can always use the leave-one-out procedure for any estimate, it is time-consuming. As a result, existing literature, e.g. \cite{hastie2019surprises}, utilize ways to approximate the leave-one-out CV procedure.

Recall that when $n_2=\infty$, the second stage of the two-stage method minimizes
\begin{eqnarray*}
    R_\epsilon({\vct\theta},\widehat{{\vct\theta}}_0(\lambda))&=&\|{\vct\theta}-\widehat{{\vct\theta}}_0(\lambda)\|^2_{\vct\Sigma}+\sigma^2+\epsilon^2\|{\vct\theta}\|^2\\
    &&+2c_0\epsilon \|{\vct\theta}\|\sqrt{\|{\vct\theta}-\widehat{{\vct\theta}}_0(\lambda)\|^2_{\vct\Sigma}+\sigma^2},
\end{eqnarray*}
and the solution is
\begin{eqnarray*}
    \widetilde{{\vct\theta}}_\epsilon(\lambda) = (\vct\Sigma+\alpha_\epsilon(\lambda) \vct{I}_d)^{-1}\vct\Sigma\widehat{\vct\theta}_0(\lambda),
\end{eqnarray*}
for some $\alpha_\epsilon(\lambda)\geq0$. One needs to rerun the CV procedure for $n_1$ times and obtain different $\vct{\widetilde\theta}_\epsilon(\lambda)^{-j}$, the leave-one-out estimate of $\vct{\widetilde\theta}_\epsilon(\lambda)$ leaving the $j$th labeled sample.

Given that the above formula $\vct{\widetilde\theta}_\epsilon(\lambda)$ is a transformation of $\widehat{\vct\theta}_0(\lambda)$, one can borrow the idea of approximating CV in clean training to the two-stage adversarial training. 
To be specific, since both the $\alpha_\epsilon(\lambda)$ and $\widehat{{\vct\theta}}_0(\lambda)$ relate to each labeled sample, assuming the $j$th sample is discarded, the estimate of the two-stage method will be
\begin{eqnarray}
    (\vct\Sigma+\alpha^{-j}\vct{I}_d)\vct\Sigma\widehat{\vct\theta}^{-j}_0(\lambda),\label{eqn:cv}
\end{eqnarray}
and we approximate both $\alpha^{-j}$ and $\widehat{\vct\theta}^{-j}_0(\lambda)$. 

The following lemma shows how to approximate $\alpha_\epsilon(\lambda)$ in the leave-one-out CV:
\begin{table*}[!ht]
    \centering
    \begin{tabular}{c|ccc}
        $\epsilon$ & 0.3 & 0.5 & 0.7\\\hline 
        Cross validation (CV loss in training) & 0.8750 & 0.9663 & 1.0300\\
        Cross validation (corresponding population risk) & 0.8871  &  0.9751 & 1.0270\\
        Cross validation for clean regression (corresponding population risk) & 0.8873  &  1.0076 & 1.1140\\\hline
        Best $\lambda$ (corresponding population risk) & 0.8741 & 0.9648 &  1.0185\\
    \end{tabular}
    \caption{Adversarial risks using cross validation and the best $\lambda$.}
    \label{tab:cross_validation}
\end{table*}

\begin{lemma}\label{thm:alpha}
Rewrite $\vct{\widetilde\theta}_\epsilon(\lambda)$ as $\vct{\widetilde\theta}$, $\vct{\widehat\theta}_0(\lambda)$ as $\vct{\widehat\theta}_0$, and $\alpha=\alpha_\epsilon(\lambda)$ for simplicity. Denote 
$\Delta_j=\widehat{\vct\theta}_0^{-j}-\widehat{\vct\theta}_0,$
and
\begin{eqnarray*}
    A_1&=&\frac{1}{\|\vct{\widetilde\theta}-\vct{\widehat\theta}_0\|_{\vct\Sigma}\|\vct{\widetilde\theta}\|}\vct{\widetilde\theta}^{\top}(\vct\Sigma+\alpha {\vct I}_d)^{-2}\vct\Sigma\vct{\widehat\theta}_0 
    \\&&- \frac{\|\vct{\widetilde\theta}\|}{\|\vct{\widetilde\theta}-\vct{\widehat\theta}_0\|_{\vct\Sigma}^3}(\vct{\widetilde\theta}-\vct{\widehat\theta}_0)^{\top}\vct\Sigma(\vct\Sigma+\alpha {\vct I}_d)^{-2}\vct\Sigma\vct{\widehat\theta}_0,\end{eqnarray*}
    \begin{eqnarray*}
    A_2&=&\frac{1}{\|\vct{\widetilde\theta}-\vct{\widehat\theta}_0\|_{\vct\Sigma}\|\vct{\widetilde\theta}\|} (\vct{\widetilde\theta}-\vct{\widehat\theta}_0)^{\top}\vct\Sigma(\vct\Sigma+\alpha {\vct I}_d)^{-2}\vct\Sigma\vct{\widehat\theta}_0 \\
    &&- \frac{\|\vct{\widetilde\theta}-\vct{\widehat\theta}_0\|_{\vct\Sigma}}{\|\vct{\widetilde\theta}\|^3}\vct{\widetilde\theta}^{\top}\vct\Sigma(\vct\Sigma+\alpha {\vct I}_d)^{-2}\vct{\widehat\theta}_0,\end{eqnarray*}
    \begin{eqnarray*}
    A_3&=&\left({I}_d+\epsilon c_0 \frac{\|\vct{\widetilde\theta}\|}{\|\vct{\widetilde\theta}-\theta_0\|_{\vct\Sigma}}\right)\Sigma(\vct\Sigma+\alpha {\vct I}_d)^{-1}\vct{\widetilde\theta}\\
    &&+\left(\epsilon c_0 \frac{\|\vct{\widetilde\theta}-\vct{\widehat\theta}_0\|_{\vct\Sigma}}{\|\vct{\widetilde\theta}\|}+\epsilon^2\right) (\vct\Sigma+\alpha {\vct I}_d)^{-1}\vct{\widetilde\theta},\end{eqnarray*}
    \begin{eqnarray*}
    A_4&=&\frac{1}{\|\vct{\widetilde\theta}-\vct{\widehat\theta}_0\|_{\vct\Sigma}\|\vct{\widetilde\theta}\|}\vct{\widetilde\theta}^{\top}(\vct\Sigma+\alpha {\vct I}_d)^{-1}\vct\Sigma  \\
    &&+\frac{\|\vct{\widetilde\theta}\|}{\|\vct{\widetilde\theta}-\vct{\widehat\theta}_0\|_{\vct\Sigma}^3}\alpha(\vct{\widetilde\theta}-\vct{\widehat\theta}_0)^{\top}\vct\Sigma(\vct\Sigma+\alpha {\vct I}_d)^{-1},\end{eqnarray*}
    \begin{eqnarray*}
    A_5&=&\frac{1}{\| \vct{\widetilde\theta}-\vct{\widehat\theta}_0\|_{\vct\Sigma}\|\vct{\widetilde\theta}\|} \alpha (\vct{\widetilde\theta}-\vct{\widehat\theta}_0)^{\top}\vct\Sigma(\vct\Sigma+\alpha {\vct I}_d)^{-1}\\
    &&-\frac{\|\vct{\widetilde\theta}-\vct{\widehat\theta}_0\|_{\vct\Sigma}}{\|\vct{\widetilde\theta}\|^3}\vct{\widetilde\theta}^{\top}(\vct\Sigma+\alpha {\vct I}_d)^{-1}\vct\Sigma,
\end{eqnarray*}
then when $\|\vct{\widehat\theta}_0-\vct{\widehat\theta}_0^{-j}\|=o(1)$, the leave-one-out estimate of $\alpha$ satisfies
\begin{eqnarray*}
    \alpha^{-j}-\alpha &=& \frac{\left( \epsilon c_0A_1\vct\Sigma(\vct{\widetilde\theta}-\vct{\widehat\theta}_0)+ \epsilon c_0 A_2\vct{\widetilde\theta}+ A_3\right)^{\top}}{\| \epsilon c_0A_1\vct\Sigma(\vct{\widetilde\theta}-\vct{\widehat\theta}_0)+ \epsilon c_0 A_2\vct{\widetilde\theta}+ A_3\|^2}\\
    &&\times\left(\epsilon c_0 A_4\Delta_j\vct\Sigma (\vct{\widetilde\theta}-\vct{\widehat\theta}_0)+\epsilon c_0 A_5\Delta_j\vct{\widetilde\theta}\right)+o,
\end{eqnarray*}
where $o$ represents negligible terms.
\end{lemma}

The proof of Lemma \ref{thm:alpha} can be found in the appendix. Based on the result in Lemma \ref{thm:alpha}, we can use
    \begin{eqnarray*}
    \widehat\alpha^{-j}-\alpha &=&\frac{\left( \epsilon c_0A_1\vct\Sigma(\vct{\widetilde\theta}-\vct{\widehat\theta}_0)+ \epsilon c_0 A_2\vct{\widetilde\theta}+ A_3\right)^{\top}}{\| \epsilon c_0A_1\vct\Sigma(\vct{\widetilde\theta}-\vct{\widehat\theta}_0)+ \epsilon c_0 A_2\vct{\widetilde\theta}+ A_3\|^2}\\
    &&\times\left(\epsilon c_0 A_4\Delta_j\vct\Sigma (\vct{\widetilde\theta}-\vct{\widehat\theta}_0)+\epsilon c_0 A_5\Delta_j\vct{\widetilde\theta}\right)
\end{eqnarray*}
to approximate $\alpha^{-j}$.

In terms of the leave-on-out estimate of $\vct{\widehat\theta}_0(\lambda)$, i.e., $\vct{\widehat\theta}_0^{-j}(\lambda)$, one can use the  Kailath Variant fomular (from 3.1.2 of \citealp{petersen2008matrix}) and obtain
\begin{eqnarray*}
\vct{\widehat{\theta}}_0(\lambda)-\vct{\widehat{\theta}}_0^{-j}(\lambda)=\frac{y_j-\widehat{y}_j(\lambda)}{1-S_j(\lambda)}(\X^{\top}\X+n\lambda \vct I_d)^{-1}\x_j,
\end{eqnarray*}
where $\X\in\mathbb{R}^{n_1\times d}$ denotes the labeled data matrix, and $\widehat{y}_j(\lambda)=\vct{\widehat{\theta}}_0(\lambda)^{\top}\x_j$ as the fitted value of the $j$th observation.

After obtaining the estimate $\widehat\alpha^{-j}$ and $\vct{\widehat\theta}_0^{-j}(\lambda)$, one can put them into (\ref{eqn:cv}) to obtain the leave-one-out estimate of $\vct{\widetilde\theta}_\epsilon(\lambda)$. The following theorem justifies the correctness of the above procedure:
\begin{theorem}\label{thm:cv}
     Denote 
\begin{eqnarray*}
    CV(\lambda,\epsilon) = \frac{1}{n_1}\sum \left( |\vct x_i^{\top }\widetilde{\vct\theta}^{-j}_\epsilon(\lambda)-y_i|+\epsilon\|\widetilde{\vct\theta}^{-j}_\epsilon(\lambda)\| \right)^2,
\end{eqnarray*} and ${\vct{\breve\theta}}^{-j}_\epsilon(\lambda)=(\vct\Sigma+\widehat\alpha^{-j}\vct{I}_d)\Sigma\widehat{\vct\theta}^{-j}_0(\lambda)$ as the approximation of the leave-one-out estimate using Lemma \ref{thm:alpha}. Then under the Gaussian model assumption of $(X,Y)$, the approximated CV converges to the actual CV result, i.e.,
\begin{eqnarray*}
    \frac{1}{n_1}\sum \left( |\vct x_i^{\top }{\vct{\breve\theta}}^{-j}_\epsilon(\lambda)-y_i|+\epsilon\|{\vct{\breve\theta}}^{-j}_\epsilon(\lambda)\| \right)^2\xrightarrow{P}
    CV(\lambda,\epsilon).
\end{eqnarray*}
% In addition, the cross validation result converges to the population adversarial risk,
% \begin{eqnarray*}
%     CV(\lambda,\epsilon)\rightarrow R_\epsilon(\vct{\widehat\theta}_\epsilon(\lambda),\vct{\theta}_0).
% \end{eqnarray*}
\end{theorem}

We use the simulation setting in Figure \ref{fig:ridge_regression} to examine the performance of the above cross validation method. The results are summarized in Table \ref{tab:cross_validation}. 

From Table \ref{tab:cross_validation}, there are two observations. First, one can see that using the cross validation, the CV loss in training is closed to the corresponding population risk. 
% This verifies the correctness of Theorem \ref{thm:cv} that $CV\rightarrow R_\epsilon$.

In addition, the performance of the proposed algorithm is closed to the optimal $\lambda$, and using clean regression in cross validation leads to a worse performance.

\section{Conclusion and Future Directions}

This paper studies the asymptotics of the two-stage adversarial training in a high-dimensional linear regression setup. Double descent is observed for the ridge-less regression case, and a better performance can be achieved via $\mathcal{L}_2$ regularization. We also derive the shortcut cross validation formula for this two-stage method to simplify the computation for cross validation.

The results in this paper can be extended in some directions. First, in literature, e.g., \cite{ba2020generalization}, the double descent phenomenon is also related to two-layer neural networks. An interesting future direction is to extend the analysis in this paper to the neural network setup. Second, since the shortcut formula for cross validation is distribution specific and assumes $n_2=\infty$, one may investigate in a more general cross validation procedure or relax to the scenario with a finite $n_2$. 

% \begin{proposition}
%     For $\lambda>0$, assuming $\|{\vct\theta}_0\|=1$, we have
%     \begin{eqnarray*}
%         CV(\lambda,0)-\sigma^2\rightarrow \sigma^2\gamma\left[m_{\gamma}(-\lambda)-\lambda\left(1-\frac{1}{\sigma^2\gamma}\lambda\right)m_{\gamma}'(-\lambda)\right].
%     \end{eqnarray*}
% \end{proposition}

% \subsection{Intuition: Why Do Unlabeled Data Improve Performance?}

% The intuition of why unlabeled data improves ridge regression is similar to how it works as in adversarial training in \cite{xing2022artificial}. 

% For a given data point $(x,y)$, if we rewrite the penalized loss function $l_{\lambda_{\text{new}}}$ as a function of ${\vct\theta}$ and ${\vct\theta}_0$, then it becomes
% $$l_{(x,y,\lambda_{\text{new}})}({\vct\theta},{\vct\theta}_0)=(x^{\top}{\vct\theta}-x^{\top}{\vct\theta}_0-(y-x^{\top}{\vct\theta}_0))^2+\lambda_{\text{new}}\|{\vct\theta}\|^2=(x^{\top}{\vct\theta}-x^{\top}{\vct\theta}_0-\varepsilon)^2+\lambda_{\text{new}}\|{\vct\theta}\|^2.$$
% When $\lambda_{\text{new}}=0$, ${\vct\theta}$ and ${\vct\theta}_0$ play the same role in $l_{(x,y,\lambda_{\text{new}})}({\vct\theta},{\vct\theta}_0)$. However, when $\lambda_{\text{new}}>0$, $l_{(x,y,\lambda_{\text{new}})}({\vct\theta},{\vct\theta}_0)$ is more sensitive to the change in ${\vct\theta}$. As a result, it is less sensitive to the change in ${\vct\theta}_0$, i.e., changing from the true model ${\vct\theta}_0$ to the estimated model $\widehat{\vct\theta}_{\lambda}$, and we can use unlabeled data to reduce the estimation variance.

% \section{Two-Stage Estimate with Extra Unlabeled Data}

\newpage
\bibliographystyle{plainnat}
\bibliography{regression}

\newpage
\appendix
\onecolumn

% \section{Preliminary Results from \cite{hastie2019surprises}}

% \section{Proofs for Two-Stage Estimate when $n_2\rightarrow\infty$}

\section{Proofs}\label{sec:proof:two}
\subsection{Theorem \ref{thm:adv_convergence}}
\begin{proof}[Proof of Theorem \ref{thm:adv_convergence}]

We first analyze $\|\vct{\widehat\theta}_0(\lambda)-\vct\theta_0\|^2$ and $\|\vct{\widehat\theta}_0(\lambda)\|^2$.

For $\|\vct{\widehat{\theta}}_0(\lambda)\|^2$, denoting $\vct y$ and $\vct\varepsilon$ as the vector of response and noise, we have
		\begin{eqnarray*}
	\|\vct{\widehat{\theta}}_0(\lambda)\|^2&=&\vct y^{\top}\X(\X^{\top}\X+\lambda n_1 \vct I_d)^{-2}\X^{\top}\vct y\\
 &=&\vct\theta_0^{\top}\X^{\top}\X(\X^{\top}\X+\lambda n_1 \vct I_d)^{-2}\X^{\top}\X \vct\theta_0 + \vct \varepsilon^{\top}\X(\X^{\top}\X+\lambda n_1 \vct I_d)^{-2}\X^{\top} \vct\varepsilon\\
 &&+2\vct\theta_0^{\top}\X^{\top}\X(\X^{\top}\X+\lambda n_1 \vct I_d)^{-2}\X^{\top}\vct\varepsilon.
	\end{eqnarray*}
	We look at the each term respectively. In probability, we have
	\begin{eqnarray*}
	&&\vct\theta_0^{\top}\X^{\top}\X(\X^{\top}\X+\lambda n_1 \vct I_d)^{-2}\X^{\top}\X \vct\theta_0 \\
 &=&r^2-2\lambda n_1\vct\theta_0^{\top}(\X^{\top}\X+\lambda n_1 \vct I_d)^{-1} \vct\theta_0+\lambda^2 n_1^2\vct\theta_0^{\top}(\X^{\top}\X+\lambda n_1 \vct I_d)^{-2} \vct\theta_0\\
	&\rightarrow& {r^2} \left[1-2\lambda m_{\gamma}(-\lambda)+\lambda^2 m_{\gamma}'(-\lambda)\right],
	\end{eqnarray*}
	and
 \begin{eqnarray*}
    \vct \varepsilon^{\top}\X(\X^{\top}\X+\lambda n_1 \vct I_d)^{-2}\X^{\top} \vct\varepsilon &\rightarrow& \sigma^2\left[\frac{1}{n_1}tr( (\widehat{\vct\Sigma}+\lambda \vct I_d)^{-1} )-\frac{1}{n_1}\lambda tr((\widehat{\vct\Sigma}+\lambda \vct I_d)^{-2} )\right]\\
     &\rightarrow& \sigma^2\left[ \gamma m_{\gamma}(-\lambda)-\lambda\gamma m_{\gamma}'(-\lambda)\right],
 \end{eqnarray*}
 where the function $m_\gamma$ is obtained from \cite{hastie2019surprises}.
    % with
    % \begin{eqnarray*}
    %     &&Var\left[\vct \varepsilon^{\top}\X(\X^{\top}\X+\lambda n_1 \vct I_d)^{-2}\X^{\top} \vct\varepsilon\right]\\
    %     &=&2tr\left[ (\X^{\top}\X+\lambda n_1 \vct I_d)^{-2}\X^{\top}\X (\X^{\top}\X+\lambda n_1 \vct I_d)^{-2}\X^{\top}\X\right]\\
    %     &=&2tr\left[ (\vct I_d-\lambda n_1(\X^{\top}\X+\lambda n_1 \vct I_d)^{-1})^2(\X^{\top}\X+\lambda n_1 \vct I_d)^{-2} \right]\\
    %     &=&2tr\left[(\X^{\top}\X+\lambda n_1 \vct I_d)^{-2} \right]-2 \lambda n_1 2tr\left[(\X^{\top}\X+\lambda n_1 \vct I_d)^{-3} \right]+\lambda^2n_1^2 2tr\left[ (\X^{\top}\X+\lambda n_1 \vct I_d)^{-4} \right]\\
    %     &\rightarrow& 0,
    % \end{eqnarray*}
 For the cross term, we also have
    \begin{eqnarray*}
        &&\left[\vct\theta_0^{\top}\X^{\top}\X(\X^{\top}\X+\lambda n_1 \vct I_d)^{-2}\X^{\top}\vct\varepsilon\right]^2\\
        &\rightarrow & \sigma^2 tr\left[ \X^{\top}\X(\X^{\top}\X+\lambda n_1 \vct I_d)^{-2}\X^{\top}\X(\X^{\top}\X+\lambda n_1 \vct I_d)^{-2} \X^{\top}\X\vct\theta_0\vct\theta_0^{\top}\right]\\
        &\xrightarrow{P}&0.
    \end{eqnarray*}
	As a result, 
	\begin{eqnarray*}
	\|\vct{\widehat{\theta}}_0(\lambda)\|^2\xrightarrow{P} r^2\left[ 1-2\lambda m_{\gamma}(-\lambda)+\lambda^2m_{\gamma}'(-\lambda) \right] +\sigma^2\gamma\left[ m_{\gamma}(-\lambda)-\lambda m_{\gamma}'(-\lambda) \right].
	\end{eqnarray*}

For $\|\vct{\widehat\theta}_0(\lambda)-\vct\theta_0\|^2$, we have
\begin{eqnarray*}
    \|\vct{\widehat\theta}_0(\lambda)-\vct\theta_0\|^2=\|\vct{\widehat{\theta}}_0(\lambda)\|^2+\|\vct{{\theta}}_0\|^2-2 \vct{\widehat\theta}_0(\lambda)^{\top}\vct{{\theta}}_0,
\end{eqnarray*}
where in probability,
\begin{eqnarray*}
    \left[\vct{\varepsilon}\X(\X^{\top}\X+\lambda n_1 \vct I_d)^{-1}\vct{{\theta}}_0\right]^2\rightarrow 0,
\end{eqnarray*}
and
\begin{eqnarray*}
    \vct{\widehat\theta}_0(\lambda)^{\top}\vct{{\theta}}_0&=&\vct{{\theta}}_0\X^{\top}\X(\X^{\top}\X+\lambda n_1 \vct I_d)^{-1}\vct{{\theta}}_0+\vct{\varepsilon}\X(\X^{\top}\X+\lambda n_1 \vct I_d)^{-1}\vct{{\theta}}_0\\
    &\rightarrow&r^2-\lambda n_1 \vct{{\theta}}_0(\X^{\top}\X+\lambda n_1 \vct I_d)^{-1}\vct{{\theta}}_0\\
    &\rightarrow& r^2 - \lambda r^2 m_{\gamma}(-\lambda).
\end{eqnarray*}
Consequently, in probability,
\begin{eqnarray*}
    \|\vct{\widehat\theta}_0(\lambda)-\vct\theta_0\|^2\rightarrow r^2\lambda^2 m'_{\gamma}(-\lambda)+\sigma^2\gamma[m_\gamma(-\lambda)-\lambda m_{\gamma}'(-\lambda)].
\end{eqnarray*}

For adversarial training, from \cite{javanmard2020precise,xing2021adversarially} we know that the minimizer of $R_\epsilon(\vct\theta, \vct{\widehat\theta}_0(\lambda))$ is
\begin{eqnarray*}
    \vct{\widetilde\theta}_\epsilon(\lambda)=(\vct\Sigma+\alpha\vct I_d)^{-1}{\vct\Sigma}\vct{\widehat\theta}_0(\lambda),
\end{eqnarray*}
where $c_0=\sqrt{2/\pi}$ and $\alpha$ satisfies
\begin{eqnarray*}
\alpha\left(1+\epsilon  c_0  \frac{\|\widetilde{{\vct\theta}}\|}{\sqrt{\|\widetilde{\vct\theta}-\widehat{\vct\theta}_0\|_{\vct\Sigma}^2+\sigma^2}}\right)= \left(\epsilon  c_0  \frac{\sqrt{\|\widetilde{\vct\theta}-\widehat{\vct\theta}_0\|_{\vct\Sigma}^2+\sigma^2}}{\|\widetilde{\vct\theta}\|}+\epsilon^2\right).
\end{eqnarray*}
When $\vct\Sigma=\vct I_d$, the above is reduced to
\begin{eqnarray*}
    \alpha+\epsilon c_0 \frac{\alpha\|\widehat{\vct\theta}_0\|}{\sqrt{ \|\widehat{\vct\theta}_0\|^2\alpha^2+\sigma^2(1+\alpha)^2 }}=\epsilon c_0 \frac{\sqrt{ \|\widehat{\vct\theta}_0\|^2\alpha^2+\sigma^2(1+\alpha)^2 }}{\|\widehat{\vct\theta}_0\|}+\epsilon^2.
\end{eqnarray*}
Since $\|\vct{\widehat{\theta}}_0(\lambda)\|^2$ asymptotically converges to some fixed value, the solution of $\alpha$ also asymptotically converges.
\end{proof}

\subsection{Cross Validation}\label{sec:proof:cv}
We present the proof of Lemma \ref{thm:alpha} and Theorem \ref{thm:cv} in this section.
\begin{proof}[Proof of Lemma \ref{thm:alpha}]
    To do cross validation, we know that
\begin{eqnarray*}
&&\widehat{{\vct\theta}}(\lambda)-\widehat{{\vct\theta}}^{-j}(\lambda)\\
&=&( \vct{X}_n^{\top}\vct{X}_n +n\lambda \vct{I}_d  )^{-1}\vct{X}_n^{\top} y\\
&&- \left[ ( \vct{X}_n^{\top}\vct{X}_n +n\lambda  \vct{I}_d)^{-1}+\frac{( \vct{X}_n^{\top}\vct{X}_n +n\lambda  \vct{I}_d)^{-1}\x_jx_j^{\top}( \vct{X}_n^{\top}\vct{X}_n +n\lambda  \vct{I}_d)^{-1}}{1-x_{j}^{\top}( \vct{X}_n^{\top}\vct{X}_n +n\lambda  \vct{I}_d)^{-1}\x_j} \right]X_{-j}^{\top} y_{-j}\\
&:=& y_j( \vct{X}_n^{\top}\vct{X}_n +n\lambda  \vct{I}_d)^{-1}\x_j- \frac{\widehat{y}_j(\vct{X}_n^{\top}\vct{X}_n+n\lambda  \vct{I}_d)^{-1}\x_j}{1-S_j(\lambda)}+ \frac{y_jS_j(\lambda)(\vct{X}_n^{\top}\vct{X}_n+n\lambda  \vct{I}_d)^{-1}\x_j}{1-S_j(\lambda)}\\
&=&\frac{y_j-\widehat{y}_j(\lambda)}{1-S_j(\lambda)}(\vct{X}_n^{\top}\vct{X}_n+n\lambda  \vct{I}_d)^{-1}\x_j.
\end{eqnarray*}
In addition, $\alpha$ satisfies
\begin{eqnarray*}
\alpha\left(1+\epsilon  c_0  \frac{\|\widetilde{{\vct\theta}}\|}{\sqrt{\|\widetilde{\vct\theta}-\widehat{\vct\theta}_0\|_{\vct\Sigma}^2+\sigma^2}}\right)= \left(\epsilon  c_0  \frac{\sqrt{\|\widetilde{\vct\theta}-\widehat{\vct\theta}_0\|_{\vct\Sigma}^2+\sigma^2}}{\|\widetilde{\vct\theta}\|}+\epsilon^2\right).
\end{eqnarray*}
For the optimal solution in the adversarial training stage, we have
	\begin{eqnarray*}
		\textbf{0}=\triangledown R_\epsilon({\vct\theta},\widehat{\vct\theta}_0)&=&2 \left[\Sigma ({\vct\theta}-\widehat{\vct\theta}_0)+\epsilon c_0 \frac{\|{\vct\theta}-\widehat{\vct\theta}_0\|_{\vct\Sigma}}{\|{\vct\theta}\|}{\vct\theta}+\epsilon c_0 \frac{\|{\vct\theta}\|}{\|{\vct\theta}-\widehat{\vct\theta}_0\|_{\vct\Sigma}}\vct\Sigma ({\vct\theta}-\widehat{\vct\theta}_0)+(\epsilon^2) {\vct\theta}\right]\\
		&=& 2 \left[\left(  \vct{I}_d+\epsilon c_0 \frac{\|{\vct\theta}\|}{\|{\vct\theta}-{\vct\theta}_0\|_{\vct\Sigma}}\right)\Sigma ({\vct\theta}-\widehat{\vct\theta}_0)+\left(\epsilon c_0 \frac{\|{\vct\theta}-\widehat{\vct\theta}_0\|_{\vct\Sigma}}{\|{\vct\theta}\|}+\epsilon^2\right){\vct\theta}\right].
	\end{eqnarray*}
 For leave-one-out CV, we have
	\begin{eqnarray*}
	\textbf{0}=	\triangledown R_\epsilon({\vct\theta},\widehat{\vct\theta}_0^{-j})&=&2 \left[\Sigma ({\vct\theta}-\widehat{\vct\theta}_0^{-j})+\epsilon c_0 \frac{\|{\vct\theta}-\widehat{\vct\theta}_0^{-j}\|_{\vct\Sigma}}{\|{\vct\theta}\|}{\vct\theta}+\epsilon c_0 \frac{\|{\vct\theta}\|}{\|{\vct\theta}-\widehat{\vct\theta}_0^{-j}\|_{\vct\Sigma}}\vct\Sigma ({\vct\theta}-\widehat{\vct\theta}_0^{-j})+(\epsilon^2) {\vct\theta}\right]\\
		&=& 2 \left[\left( \vct{I}_d+\epsilon c_0 \frac{\|{\vct\theta}\|}{\|{\vct\theta}-{\vct\theta}_0^{-j}\|_{\vct\Sigma}}\right)\Sigma ({\vct\theta}-\widehat{\vct\theta}_0^{-j})+\left(\epsilon c_0 \frac{\|{\vct\theta}-\widehat{\vct\theta}_0^{-j}\|_{\vct\Sigma}}{\|{\vct\theta}\|}+\epsilon^2\right){\vct\theta}\right].
	\end{eqnarray*}
 Consequently,
\begin{eqnarray*}
    &&\left( \vct{I}_d+\epsilon c_0 \frac{\|\widetilde{\vct\theta}\|}{\|\widetilde{\vct\theta}-{\vct\theta}_0\|_{\vct\Sigma}}\right)\Sigma (\widetilde{\vct\theta}-\widehat{\vct\theta}_0)+\left(\epsilon c_0 \frac{\|\widetilde{\vct\theta}-\widehat{\vct\theta}_0\|_{\vct\Sigma}}{\|\widetilde{\vct\theta}\|}+\epsilon^2\right)\widetilde{\vct\theta}\\
    &=&\left(  \vct{I}_d+\epsilon c_0 \frac{\|\widetilde{\vct\theta}^{-j}\|}{\|\widetilde{\vct\theta}^{-j}-{\vct\theta}_0^{-j}\|_{\vct\Sigma}}\right)\Sigma (\widetilde{\vct\theta}^{-j}-\widehat{\vct\theta}_0^{-j})+\left(\epsilon c_0 \frac{\|\widetilde{\vct\theta}^{-j}-\widehat{\vct\theta}_0^{-j}\|_{\vct\Sigma}}{\|\widetilde{\vct\theta}^{-j}\|}+\epsilon^2\right)\widetilde{\vct\theta}^{-j}.
\end{eqnarray*}
Denote
\begin{eqnarray*}
    \Delta_j&=&\widehat{\vct\theta}_0^{-j}-\widehat{\vct\theta}_0,
\end{eqnarray*}
and denote $\alpha^{-j}$ as the best $\alpha$ without $j$th sample. Then
\begin{eqnarray*}
    \widetilde{\vct\theta}^{-j}-\widetilde{{\vct\theta}}&=&(\vct\Sigma+\alpha^{-j}\vct{I}_d)^{-1}\vct\Sigma\widehat{\vct\theta}_0^{-j}-(\vct\Sigma+\alpha \vct{I}_d)^{-1}\vct\Sigma\widehat{\vct\theta}_0\\
    &=&(\vct\Sigma+\alpha \vct{I}_d)^{-1}\vct\Sigma \Delta_j - (\alpha^{-j}-\alpha)(\vct\Sigma+\alpha \vct{I}_d)^{-1}\widetilde{\vct\theta}+R_0.
\end{eqnarray*}
When $\|\vct{\widehat\theta}_0\|$ and $\|\vct{\widetilde\theta}\|$ are away from zero,  $$\|R_0\|=O(|\alpha^{-j}-\alpha|\|\Delta_j\|).$$
As a result,
\begin{eqnarray*}
    &&\left( \vct{I}_d+\epsilon c_0 \frac{\|\widetilde{\vct\theta}\|}{\|\widetilde{\vct\theta}-{\vct\theta}_0\|_{\vct\Sigma}}\right)\Sigma (\widetilde{\vct\theta}-\widehat{\vct\theta}_0)+\left(\epsilon c_0 \frac{\|\widetilde{\vct\theta}-\widehat{\vct\theta}_0\|_{\vct\Sigma}}{\|\widetilde{\vct\theta}\|}+\epsilon^2\right)\widetilde{\vct\theta}\\
    &=&\left(  \vct{I}_d+\epsilon c_0 \frac{\|\widetilde{\vct\theta}^{-j}\|}{\|\widetilde{\vct\theta}^{-j}-{\vct\theta}_0^{-j}\|_{\vct\Sigma}}\right)\Sigma (\widetilde{\vct\theta}^{-j}-\widehat{\vct\theta}_0^{-j})+\left(\epsilon c_0 \frac{\|\widetilde{\vct\theta}^{-j}-\widehat{\vct\theta}_0^{-j}\|_{\vct\Sigma}}{\|\widetilde{\vct\theta}^{-j}\|}+\epsilon^2\right)\widetilde{\vct\theta}^{-j}\\
    &&+\left(  \vct{I}_d+\epsilon c_0 \frac{\|\widetilde{\vct\theta}\|}{\|\widetilde{\vct\theta}-{\vct\theta}_0\|_{\vct\Sigma}}\right)\Sigma (\widetilde{\vct\theta}^{-j}-\widehat{\vct\theta}_0^{-j})+\left(\epsilon c_0 \frac{\|\widetilde{\vct\theta}-\widehat{\vct\theta}_0\|_{\vct\Sigma}}{\|\widetilde{\vct\theta}\|}+\epsilon^2\right)\widetilde{\vct\theta}^{-j}\\
    &&-\left(  \vct{I}_d+\epsilon c_0 \frac{\|\widetilde{\vct\theta}\|}{\|\widetilde{\vct\theta}-{\vct\theta}_0\|_{\vct\Sigma}}\right)\Sigma (\widetilde{\vct\theta}^{-j}-\widehat{\vct\theta}_0^{-j})-\left(\epsilon c_0 \frac{\|\widetilde{\vct\theta}-\widehat{\vct\theta}_0\|_{\vct\Sigma}}{\|\widetilde{\vct\theta}\|}+\epsilon^2\right)\widetilde{\vct\theta}^{-j}\\
    &=&\epsilon c_0 \left(  \frac{\|\widetilde{\vct\theta}^{-j}\|}{\|\widetilde{\vct\theta}^{-j}-{\vct\theta}_0^{-j}\|_{\vct\Sigma}}-\frac{\|\widetilde{\vct\theta}\|}{\|\widetilde{\vct\theta}-{\vct\theta}_0\|_{\vct\Sigma}}\right)\Sigma (\widetilde{\vct\theta}^{-j}-\widehat{\vct\theta}_0^{-j})+\epsilon c_0 \left(\frac{\|\widetilde{\vct\theta}^{-j}-\widehat{\vct\theta}_0^{-j}\|_{\vct\Sigma}}{\|\widetilde{\vct\theta}^{-j}\|}-\frac{\|\widetilde{\vct\theta}-\widehat{\vct\theta}_0\|_{\vct\Sigma}}{\|\widetilde{\vct\theta}\|}\right)\widetilde{\vct\theta}^{-j}\\
    &&+\left(  \vct{I}_d+\epsilon c_0 \frac{\|\widetilde{\vct\theta}\|}{\|\widetilde{\vct\theta}-{\vct\theta}_0\|_{\vct\Sigma}}\right)\Sigma (\widetilde{\vct\theta}^{-j}-\widehat{\vct\theta}_0^{-j})+\left(\epsilon c_0 \frac{\|\widetilde{\vct\theta}-\widehat{\vct\theta}_0\|_{\vct\Sigma}}{\|\widetilde{\vct\theta}\|}+\epsilon^2\right)\widetilde{\vct\theta}^{-j},
\end{eqnarray*}
and changing the order of the terms in the above, we have
\begin{eqnarray*}
   &&\left(  \vct{I}_d+\epsilon c_0 \frac{\|\widetilde{\vct\theta}\|}{\|\widetilde{\vct\theta}-{\vct\theta}_0\|_{\vct\Sigma}}\right)\Sigma (\widetilde{\vct\theta}^{-j}-\widehat{\vct\theta}_0^{-j})+\left(\epsilon c_0 \frac{\|\widetilde{\vct\theta}-\widehat{\vct\theta}_0\|_{\vct\Sigma}}{\|\widetilde{\vct\theta}\|}+\epsilon^2\right)\widetilde{\vct\theta}^{-j}-\triangledown R_\epsilon({\vct{\widetilde\theta}},\widehat{\vct\theta}_0)\\
   &=&
   \left( \vct{I}_d+\epsilon c_0 \frac{\|\widetilde{\vct\theta}\|}{\|\widetilde{\vct\theta}-{\vct\theta}_0\|_{\vct\Sigma}}\right)\Sigma (\widetilde{\vct\theta}^{-j}-\widetilde{\vct\theta}-\Delta_j)+\left(\epsilon c_0 \frac{\|\widetilde{\vct\theta}-\widehat{\vct\theta}_0\|_{\vct\Sigma}}{\|\widetilde{\vct\theta}\|}+\epsilon^2\right)(\widetilde{\vct\theta}^{-j}-\widetilde{\vct\theta})\\ 
   &=&-\epsilon c_0 \left(  \frac{\|\widetilde{\vct\theta}^{-j}\|}{\|\widetilde{\vct\theta}^{-j}-{\vct\theta}_0^{-j}\|_{\vct\Sigma}}-\frac{\|\widetilde{\vct\theta}\|}{\|\widetilde{\vct\theta}-\widehat{\vct\theta}_0\|_{\vct\Sigma}}\right)\Sigma (\widetilde{\vct\theta}-\widehat{\vct\theta}_0)-\epsilon c_0 \left(\frac{\|\widetilde{\vct\theta}^{-j}-\widehat{\vct\theta}_0^{-j}\|_{\vct\Sigma}}{\|\widetilde{\vct\theta}^{-j}\|}-\frac{\|\widetilde{\vct\theta}-\widehat{\vct\theta}_0\|_{\vct\Sigma}}{\|\widetilde{\vct\theta}\|}\right)\widetilde{\vct\theta}+R_1,
\end{eqnarray*}
for $$\|R_1\|=O(\|\vct{\widetilde\theta}^{-j}-\vct{\widehat\theta}_0\|\|\Delta_j\|)=O(\|\Delta_j\|^2+|\alpha^{-j}-\alpha|\|\Delta_j\|).$$
We know that
\begin{eqnarray*}
    &&\frac{\|\widetilde{\vct\theta}^{-j}\|}{\|\widetilde{\vct\theta}^{-j}-\widehat{\vct\theta}_0^{-j}\|_{\vct\Sigma}}-\frac{\|\widetilde{\vct\theta}\|}{\|\widetilde{\vct\theta}-\widehat{\vct\theta}_0\|_{\vct\Sigma}}\\
    &=&\frac{1}{\|\widetilde{\vct\theta}-\widehat{\vct\theta}_0\|_{\vct\Sigma}}\frac{ \widetilde{\vct\theta}^{\top}(\widetilde{\vct\theta}^{-j}-\widetilde{\vct\theta})}{\|\widetilde{\vct\theta}\|}-\frac{(\widetilde{\vct\theta}-\widehat{\vct\theta}_0)^{\top}\vct\Sigma(\widetilde{\vct\theta}^{-j}-\widetilde{\vct\theta} - \Delta_j)}{\|\widetilde{\vct\theta}-\widehat{\vct\theta}_0\|_{\vct\Sigma}^3}\|\widetilde{\vct\theta}\| +O(\|R_0\|),
    \end{eqnarray*}
    and rewriting $\vct{\widetilde\theta}$ and $\vct{\widetilde\theta}^{-j}$ as functions of $\alpha$ and $\alpha^{-j}$, the first-order terms can be represented as
    \begin{eqnarray*}
    && \frac{1}{\|\widetilde{\vct\theta}-\widehat{\vct\theta}_0\|_{\vct\Sigma}}\frac{ \widetilde{\vct\theta}^{\top}(\widetilde{\vct\theta}^{-j}-\widetilde{\vct\theta})}{\|\widetilde{\vct\theta}\|}-\frac{(\widetilde{\vct\theta}-\widehat{\vct\theta}_0)^{\top}\vct\Sigma(\widetilde{\vct\theta}^{-j}-\widetilde{\vct\theta} - \Delta_j)}{\|\widetilde{\vct\theta}-\widehat{\vct\theta}_0\|_{\vct\Sigma}^3}\|\widetilde{\vct\theta}\|\\
    &=&\frac{1}{\|\widetilde{\vct\theta}-\widehat{\vct\theta}_0\|_{\vct\Sigma}}\frac{ \widetilde{\vct\theta}^{\top}\left((\vct\Sigma+\alpha \vct{I}_d)^{-1}\vct\Sigma \Delta_j - (\alpha^{-j}-\alpha)(\vct\Sigma+\alpha \vct{I}_d)^{-2}\vct\Sigma\widehat{\vct\theta}_0\right)}{\|\widetilde{\vct\theta}\|}\\
    &&-\frac{(\widetilde{\vct\theta}-\widehat{\vct\theta}_0)^{\top}\vct\Sigma\left((\vct\Sigma+\alpha \vct{I}_d)^{-1}\vct\Sigma \Delta_j - (\alpha^{-j}-\alpha)(\vct\Sigma+\alpha \vct{I}_d)^{-2}\vct\Sigma\widehat{\vct\theta}_0 - \Delta_j\right)}{\|\widetilde{\vct\theta}-\widehat{\vct\theta}_0\|_{\vct\Sigma}^3}\|\widetilde{\vct\theta}\|+O(\|R_0\|)\\
    &=&\frac{1}{\|\widetilde{\vct\theta}-\widehat{\vct\theta}_0\|_{\vct\Sigma}\|\widetilde{\vct\theta}\|}\left( \widetilde{\vct\theta}^{\top}(\vct\Sigma+\alpha \vct{I}_d)^{-1}\vct\Sigma \Delta_j -(\alpha^{-j}-\alpha)\widetilde{\vct\theta}^{\top}(\vct\Sigma+\alpha \vct{I}_d)^{-2}\vct\Sigma\widehat{\vct\theta}_0\right)+O(\|R_0\|)\\
    &&-\frac{\|\widetilde{\vct\theta}\|}{\|\widetilde{\vct\theta}-\widehat{\vct\theta}_0\|_{\vct\Sigma}^3}\left( (\widetilde{\vct\theta}-\widehat{\vct\theta}_0)^{\top}\vct\Sigma(\vct\Sigma+\alpha \vct{I}_d)^{-1}\vct\Sigma\Delta_j-(\alpha^{-j}-\alpha)(\widetilde{\vct\theta}-\widehat{\vct\theta}_0)^{\top}\vct\Sigma(\vct\Sigma+\alpha \vct{I}_d)^{-2}\vct\Sigma\widehat{\vct\theta}_0-(\widetilde{\vct\theta}-\widehat{\vct\theta}_0)^{\top}\vct\Sigma\Delta_j \right)\\
    &=&\underbrace{\frac{1}{\|\widetilde{\vct\theta}-\widehat{\vct\theta}_0\|_{\vct\Sigma}\|\widetilde{\vct\theta}\|}\alpha\widetilde{\vct\theta}^{\top}(\vct\Sigma+\alpha \vct{I}_d)^{-2}\vct\Sigma\widehat{\vct\theta}_0-\frac{\|\widetilde{\vct\theta}\|}{\|\widetilde{\vct\theta}-\widehat{\vct\theta}_0\|_{\vct\Sigma}^3}\alpha(\widetilde{\vct\theta}-\widehat{\vct\theta}_0)^{\top}\vct\Sigma(\vct\Sigma+\alpha \vct{I}_d)^{-2}\vct\Sigma\widehat{\vct\theta}_0}_{:=A_1\alpha}\\
    &&+\underbrace{\left(\frac{1}{\|\widetilde{\vct\theta}-\widehat{\vct\theta}_0\|_{\vct\Sigma}\|\widetilde{\vct\theta}\|}\widetilde{\vct\theta}^{\top}(\vct\Sigma+\alpha \vct{I}_d)^{-1}\vct\Sigma  +\frac{\|\widetilde{\vct\theta}\|}{\|\widetilde{\vct\theta}-\widehat{\vct\theta}_0\|_{\vct\Sigma}^3}\alpha(\widetilde{\vct\theta}-\widehat{\vct\theta}_0)^{\top}\vct\Sigma(\vct\Sigma+\alpha \vct{I}_d)^{-1}\right)}_{:=A_4}\Delta_j\\
    &&-\alpha^{-j}\underbrace{\left(\frac{1}{\|\widetilde{\vct\theta}-\widehat{\vct\theta}_0\|_{\vct\Sigma}\|\widetilde{\vct\theta}\|}\widetilde{\vct\theta}^{\top}(\vct\Sigma+\alpha \vct{I}_d)^{-2}\vct\Sigma\widehat{\vct\theta}_0 - \frac{\|\widetilde{\vct\theta}\|}{\|\widetilde{\vct\theta}-\widehat{\vct\theta}_0\|_{\vct\Sigma}^3}(\widetilde{\vct\theta}-\widehat{\vct\theta}_0)^{\top}\vct\Sigma(\vct\Sigma+\alpha \vct{I}_d)^{-2}\vct\Sigma\widehat{\vct\theta}_0\right)}_{=A_2}+O(\|R_0\|).
\end{eqnarray*}

Similarly,
\begin{eqnarray*}
    &&\frac{\|\widetilde{\vct\theta}^{-j}-\widehat{\vct\theta}_0^{-j}\|_{\vct\Sigma}}{\|\widetilde{\vct\theta}^{-j}\|}-\frac{\|\widetilde{\vct\theta}-\widehat{\vct\theta}_0\|_{\vct\Sigma}}{\|\widetilde{\vct\theta}\|}\\
    &=&\frac{(\widetilde{\vct\theta}-\widehat{\vct\theta}_0)^{\top}\vct\Sigma(\widetilde{\vct\theta}^{-j}-\widetilde{\vct\theta}-\Delta_j)}{\|\widetilde{\vct\theta}\|\|\widetilde{\vct\theta}-\widehat{\vct\theta}_0\|_{\vct\Sigma}}-\frac{\|\widetilde{\vct\theta}-\widehat{\vct\theta}_0\|_{\vct\Sigma}\widetilde{\vct\theta}^{\top}(\widetilde{\vct\theta}^{-j}-\widetilde{\vct\theta})}{\|\widetilde{\vct\theta}\|^3}\\
    % &&-\frac{1}{2}\frac{\|\widetilde{\vct\theta}^{-j}-\widetilde{\vct\theta}\|^2}{\|\widetilde{\vct\theta}\|^3}\|\widetilde{\vct\theta}-\widehat{\vct\theta}_0\|_{\vct\Sigma}+\frac{3}{4}\frac{\left[\widetilde{\vct\theta}^{\top}(\widetilde{\vct\theta}^{-j}-\widetilde{\vct\theta})\right]^2}{\|\widetilde{\vct\theta}\|^5}\|\widetilde{\vct\theta}-\widehat{\vct\theta}_0\|_{\vct\Sigma}\\
    % &&+\frac{1}{2}\frac{\|\widetilde{\vct\theta}^{-j}-\widetilde{\vct\theta}-\Delta_j\|_{\vct\Sigma}^2}{\|\widetilde{\vct\theta}\|\|\widetilde{\vct\theta}-\widehat{\vct\theta}_0\|_{\vct\Sigma}}- \frac{1}{2}\frac{\left[(\widetilde{\vct\theta}-\widehat{\vct\theta}_0)^{\top}\vct\Sigma(\widetilde{\vct\theta}^{-j}-\widetilde{\vct\theta}-\Delta_j)\right]^2}{\|\widetilde{\vct\theta}\|\|\widetilde{\vct\theta}-\widehat{\vct\theta}_0\|_{\vct\Sigma}^3}+o \\
    &=&\underbrace{\frac{1}{\| \widetilde{\vct\theta}-\widehat{\vct\theta}_0\|_{\vct\Sigma}\|\widetilde{\vct\theta}\|}\alpha (\widetilde{\vct\theta}-\widehat{\vct\theta}_0)^{\top}\vct\Sigma(\vct\Sigma+\alpha \vct{I}_d)^{-2}\vct\Sigma\widehat{\vct\theta}_0-\frac{\|\widetilde{\vct\theta}-\widehat{\vct\theta}_0\|_{\vct\Sigma}}{\|\widetilde{\vct\theta}\|^3}\alpha\widetilde{\vct\theta}^{\top}(\vct\Sigma+\alpha \vct{I}_d)^{-2}\vct\Sigma\widehat{\vct\theta}_0 }_{:=A_2\alpha}\\
    &&+\underbrace{\left(\frac{1}{\| \widetilde{\vct\theta}-\widehat{\vct\theta}_0\|_{\vct\Sigma}\|\widetilde{\vct\theta}\|} \alpha (\widetilde{\vct\theta}-\widehat{\vct\theta}_0)^{\top}\vct\Sigma(\vct\Sigma+\alpha \vct{I}_d)^{-1}-\frac{\|\widetilde{\vct\theta}-\widehat{\vct\theta}_0\|_{\vct\Sigma}}{\|\widetilde{\vct\theta}\|^3}\widetilde{\vct\theta}^{\top}(\vct\Sigma+\alpha \vct{I}_d)^{-1}\vct\Sigma \right)}_{:=A_5}\Delta_j\\
    &&-\alpha^{-j}\underbrace{\left(\frac{1}{\|\widetilde{\vct\theta}-\widehat{\vct\theta}_0\|_{\vct\Sigma}\|\widetilde{\vct\theta}\|} (\widetilde{\vct\theta}-\widehat{\vct\theta}_0)^{\top}\vct\Sigma(\vct\Sigma+\alpha \vct{I}_d)^{-2}\vct\Sigma\widehat{\vct\theta}_0 - \frac{\|\widetilde{\vct\theta}-\widehat{\vct\theta}_0\|_{\vct\Sigma}}{\|\widetilde{\vct\theta}\|^3}\widetilde{\vct\theta}^{\top}\vct\Sigma(\vct\Sigma+\alpha \vct{I}_d)^{-2}\widehat{\vct\theta}_0\right)}_{=A_2}+O(\|R_0\|).
\end{eqnarray*}
As a result,
\begin{eqnarray*}
   &&\left( \vct{I}_d+\epsilon c_0 \frac{\|\widetilde{\vct\theta}\|}{\|\widetilde{\vct\theta}-{\vct\theta}_0\|_{\vct\Sigma}}\right)\Sigma (\widetilde{\vct\theta}^{-j}-\widetilde{\vct\theta}-\Delta_j)+\left(\epsilon c_0 \frac{\|\widetilde{\vct\theta}-\widehat{\vct\theta}_0\|_{\vct\Sigma}}{\|\widetilde{\vct\theta}\|}+\epsilon^2\right)(\widetilde{\vct\theta}^{-j}-\widetilde{\vct\theta})\\ 
   &=&-\frac{1}{2}\left(\epsilon c_0 \left( A_1+A_2\Delta_j-\alpha^{-j} A_3\right)\Sigma (\widetilde{\vct\theta}-\widehat{\vct\theta}_0)+\epsilon c_0 \left(A_4+A_5\Delta_j-\alpha^{-j} A_6\right)\widetilde{\vct\theta}\right)+O(\|R_0\|+\|R_0\|)\\
   &=&\left(\vct{I}_d+\epsilon c_0 \frac{\|\widetilde{\vct\theta}\|}{\|\widetilde{\vct\theta}-{\vct\theta}_0\|_{\vct\Sigma}}\right)\Sigma\left((\vct\Sigma+\alpha \vct{I}_d)^{-1}\vct\Sigma \Delta_j - (\alpha^{-j}-\alpha)(\vct\Sigma+\alpha \vct{I}_d)^{-2}\vct\Sigma\widehat{\vct\theta}_0-\Delta_j\right)\\
   &&+\left(\epsilon c_0 \frac{\|\widetilde{\vct\theta}-\widehat{\vct\theta}_0\|_{\vct\Sigma}}{\|\widetilde{\vct\theta}\|}+\epsilon^2\right)\left((\vct\Sigma+\alpha \vct{I}_d)^{-1}\vct\Sigma \Delta_j - (\alpha^{-j}-\alpha)(\vct\Sigma+\alpha \vct{I}_d)^{-2}\vct\Sigma\widehat{\vct\theta}_0\right)\\
   &=&
   \underbrace{\left[-\alpha \left(\vct{I}_d+\epsilon c_0 \frac{\|\widetilde{\vct\theta}\|}{\|\widetilde{\vct\theta}-{\vct\theta}_0\|_{\vct\Sigma}}\right)+\left(\epsilon c_0 \frac{\|\widetilde{\vct\theta}-\widehat{\vct\theta}_0\|_{\vct\Sigma}}{\|\widetilde{\vct\theta}\|}+\epsilon^2\right)\right](\vct\Sigma+\alpha \vct{I}_d)^{-1}\vct\Sigma}_{=\textbf{0}}\Delta_j\\
   &&+\alpha\underbrace{ \left[\left(\vct{I}_d+\epsilon c_0 \frac{\|\widetilde{\vct\theta}\|}{\|\widetilde{\vct\theta}-{\vct\theta}_0\|_{\vct\Sigma}}\right)\Sigma+\left(\epsilon c_0 \frac{\|\widetilde{\vct\theta}-\widehat{\vct\theta}_0\|_{\vct\Sigma}}{\|\widetilde{\vct\theta}\|}+\epsilon^2\right)\right](\vct\Sigma+\alpha \vct{I}_d)^{-1}\widetilde{\vct\theta}}_{:=A_3}\\
   &&-\alpha^{-j}\left[\left(\vct{I}_d+\epsilon c_0 \frac{\|\widetilde{\vct\theta}\|}{\|\widetilde{\vct\theta}-{\vct\theta}_0\|_{\vct\Sigma}}\right)\Sigma+\left(\epsilon c_0 \frac{\|\widetilde{\vct\theta}-\widehat{\vct\theta}_0\|_{\vct\Sigma}}{\|\widetilde{\vct\theta}\|}+\epsilon^2\right) \right](\vct\Sigma+\alpha \vct{I}_d)^{-1}\widetilde{\vct\theta},
\end{eqnarray*}
that is,
\begin{eqnarray*}
    &&-\epsilon c_0 A_1\Delta_j\vct\Sigma (\widetilde{\vct\theta}-\widehat{\vct\theta}_0)-\epsilon c_0\alpha A_3\Delta_j\widetilde{\vct\theta} \\
    &=& (\alpha^{-j}-\alpha) \left( \epsilon c_0A_2\Sigma(\widetilde{\vct\theta}-\widehat{\vct\theta}_0)+ \epsilon c_0A_4\widetilde{\vct\theta}- A_5  \right)+O(\|R_0\|+\|R_0\|),
\end{eqnarray*}
and
\begin{eqnarray*}
    \alpha^{-j}-\alpha &\approx& \frac{\left( \epsilon c_0A_2\Sigma(\widetilde{\vct\theta}-\widehat{\vct\theta}_0)+ \epsilon c_0 A_4\widetilde{\vct\theta}+ A_5\right)^{\top}}{\| \epsilon c_0A_2\Sigma(\widetilde{\vct\theta}-\widehat{\vct\theta}_0)+ \epsilon c_0 A_4\widetilde{\vct\theta}+ A_5\|^2}\left(\epsilon c_0 A_1\Delta_j\vct\Sigma (\widetilde{\vct\theta}-\widehat{\vct\theta}_0)+\epsilon c_0A_3\Delta_j\widetilde{\vct\theta} \right).
\end{eqnarray*}
\end{proof}

% \begin{theorem}
%     Denote 
% \begin{eqnarray*}
%     CV(\lambda,\epsilon) = \frac{1}{n_1}\sum \left( |\vct x_i^{\top }\widetilde{\vct\theta}^{-j}_\epsilon(\lambda)-y_i|+\epsilon\|\widetilde{\vct\theta}^{-j}_\epsilon(\lambda)\| \right)^2,
% \end{eqnarray*}
% and ${\vct{\breve\theta}}^{-j}_\epsilon(\lambda)=(\vct\Sigma+\widehat\alpha^{-j}\vct{I}_d)\Sigma\widehat{\vct\theta}^{-j}_0(\lambda)$, then
% \begin{eqnarray*}
%     \frac{1}{n_1}\sum \left( |\vct x_i^{\top }{\vct{\breve\theta}}^{-j}_\epsilon(\lambda)-y_i|+\epsilon\|{\vct{\breve\theta}}^{-j}_\epsilon(\lambda)\| \right)^2\rightarrow CV(\lambda,\epsilon).
% \end{eqnarray*}
% \end{theorem}
\begin{proof}[Proof of Theorem \ref{thm:cv}]
    From Lemma \ref{thm:alpha}, we know that when $\|\Delta_j\|=o(1)$, $\alpha^{-j}-\alpha=o(1)$. In this proof, we check whether $\|\Delta_j\|\rightarrow 0$ for all $j=1,\ldots,n_1$.
    
One can use the  Kailath Variant fomular (from 3.1.2 of \cite{petersen2008matrix}) to obtain
\begin{eqnarray*}
&&\vct{\widehat{\theta}}_0(\lambda)-\vct{\widehat{\theta}}_0^{-j}(\lambda)\\
&=&( \X^{\top}\X +n\lambda \vct I_d)^{-1}\X^{\top}\vct y\\
&&- \left[ ( \X^{\top}\X +n\lambda \vct I_d)^{-1}+\frac{( \X^{\top}\X +n\lambda \vct I_d)^{-1}\x_j\x_j^{\top}( \X^{\top}\X +n\lambda \vct I_d)^{-1}}{1-\x_{j}^{\top}( \X^{\top}\X +n\lambda \vct I_d)^{-1}\x_j} \right]\X_{-j}^{\top}\vct y_{-j}\\
&=& y_j( \X^{\top}\X +n\lambda \vct I_d)^{-1}\x_j- \frac{\widehat{y}_j(\X^{\top}\X+n\lambda \vct I_d)^{-1}\x_j}{1-S_j(\lambda)}+ \frac{y_jS_j(\lambda)(\X^{\top}\X+n\lambda \vct I_d)^{-1}\x_j}{1-S_j(\lambda)}\\
&=&\frac{y_j-\widehat{y}_j(\lambda)}{1-S_j(\lambda)}(\X^{\top}\X+n\lambda \vct I_d)^{-1}\x_j,
\end{eqnarray*}
where $\widehat{y}_j(\lambda)=\vct{\widehat{\theta}}_0(\lambda)^{\top}\x_j$.

Based on \cite{hastie2019surprises}, almost surely,
In addition, denote $\vct A_i=n_1(\X_{-i}^{\top}\X_{-i}+\lambda n_1 \vct I_d)^{-1}$, and $\delta_i=\frac{\x_i}{\sqrt{n_1}}$. Then
	\begin{eqnarray*}
	\vct x_i^{\top }(\X^{\top}\X+\lambda n_1 \vct I_d)^{-2}\vct x_i
	&=&\frac{1}{n_1}\delta_i\left(\vct A_i-\frac{\vct A_i\delta_i\delta_i^{\top}\vct A_i}{1+\delta_i^{\top}\vct A_i\delta_i} \right)^2\delta_i\\
	&=&\frac{1}{n_1}\delta_i\left(\vct A_i^2-2\frac{\vct A_i^2\delta_i\delta_i^{\top}\vct A_i}{1+\delta_i^{\top}\vct A_i\delta_i} +\frac{\vct A_i\delta_i\delta_i^{\top}\vct A_i^2\delta_i\delta_i^{\top}\vct A_i}{(1+\delta_i^{\top}\vct A_i\delta_i)^2 } \right)\delta_i\\
	&=&\frac{1}{n_1}\left( \delta_i\vct A_i^2\delta_i-2\frac{\delta_i^{\top}\vct A_i^2\delta_i\delta_i^{\top}\vct A_i\delta_i}{1+\delta_i^{\top}\vct A_i\delta_i}+\frac{ (\delta_i^{\top}\vct A_i\delta_i)^2\delta_i^{\top}\vct A_i^2\delta_i }{(1+\delta_i^{\top}\vct A_i\delta_i)^2} \right)\\
	&=&\frac{1}{n_1}\frac{\delta_i^{\top}\vct A_i^2\delta_i}{(1+\delta_i^{\top}\vct A_i\delta_i)^2}
	\\&\xrightarrow{a.s.}& \frac{1}{n_1}\frac{  \gamma m_{\gamma}'(-\lambda) }{(1+\gamma m_{\gamma}(-\lambda))^2}.
	\end{eqnarray*}
 Finally, for $y_i-\widehat{y}_i$, we have
\begin{eqnarray*}
	y_i-\widehat{y}_i&=& y_i -\vct x_i^{\top }(\X^{\top}\X+\lambda n_1 \vct I_d)^{-1}\X^{\top}(\X \vct\theta_0+\vct\varepsilon)\\
	&=&\varepsilon_i-\vct x_i^{\top }(\X^{\top}\X+\lambda n_1 \vct I_d)^{-1}\X^{\top}\vct\varepsilon +\lambda n_1\vct x_i^{\top }(\X^{\top}\X+\lambda n_1 \vct I_d)^{-1}\vct{\theta}_0.
	\end{eqnarray*}
Using Sherman–Morrison formula, we have
\begin{eqnarray*}
    \vct x_i^{\top }(\X^{\top}\X+\lambda n_1 \vct I_d)^{-1}\vct{\theta}_0 =  \vct x_i^{\top }\left[\frac{1}{n_1} \vct A_i-\frac{\vct A_i\vct x_i\vct x_i^{\top } \vct A_i/n_1^2 }{1+\vct x_i^{\top }\vct A_i\vct x_i/n_1}\right]\vct{\theta}_0,
\end{eqnarray*}
thus
\begin{eqnarray*}
    &&\left(\vct x_i^{\top }(\X^{\top}\X+\lambda n_1 \vct I_d)^{-1}\vct{\theta}_0\right)^2\\
    &=&\vct x_i^{\top }\frac{1}{n_1^2} \vct A_i \vct{\theta}_0\vct{\theta}_0^{\top}\vct A_i\vct x_i - \frac{2}{n_1} \vct x_i^{\top } \vct A_i\vct\theta_0\vct\theta_0 \frac{\vct A_i\vct x_i\vct x_i^{\top } \vct A_i/n_1^2 }{1+\vct x_i^{\top }\vct A_i\vct x_i/n_1}\vct x_i +  \vct x_i^{\top } \frac{\vct A_i\vct x_i\vct x_i^{\top } \vct A_i/n_1^2 }{1+\vct x_i^{\top }\vct A_i\vct x_i/n_1}\vct\theta_0\vct\theta_0 \frac{\vct A_i\vct x_i\vct x_i^{\top } \vct A_i/n_1^2 }{1+\vct x_i^{\top }\vct A_i\vct x_i/n_1}\vct x_i\\
    &=& O_p(tr\left( \vct A_i^2 \vct{\theta}_0\vct{\theta}_0^{\top} \right)/n_1^2)\\
    &=& O_p\left( m_\gamma(-\lambda)/n_1^2\right).
    % -\frac{2\gamma}{n_1^2}\frac{\gamma m_\gamma(-\lambda) }{1+\gamma m_\gamma(-\lambda)}  tr\left( \vct A_i^2\vct\theta_0\vct\theta_0^{\top} \right)  
\end{eqnarray*}
In addition,
\begin{eqnarray*}
    \left(\vct x_i^{\top }(\X^{\top}\X+\lambda n_1 \vct I_d)^{-1}\X^{\top}\vct\varepsilon \right)^2\rightarrow \sigma^2 \vct x_i^{\top }(\X^{\top}\X+\lambda n_1 \vct I_d)^{-1}\X^{\top}\X(\X^{\top}\X+\lambda n_1 \vct I_d)^{-1} \vct x_i,
\end{eqnarray*}
which converges to a constant.

Finally, given the distribution of $\varepsilon$, we have with probability tending to 1,
$$\sup_{j}\|\vct{\widehat{\theta}}_0(\lambda)-\vct{\widehat{\theta}}_0^{-j}(\lambda)\|^2=o((\log n_1)/n_1).$$

\end{proof}
\end{document}